\documentclass{article} 

\usepackage{iclr2027_conference,times}

\usepackage{amsmath,amsfonts,bm}

\def\eqref#1{equation~\ref{#1}}

\def\1{\bm{1}}

\DeclareMathAlphabet{\mathsfit}{\encodingdefault}{\sfdefault}{m}{sl}
\SetMathAlphabet{\mathsfit}{bold}{\encodingdefault}{\sfdefault}{bx}{n}

\usepackage{amsmath}
\usepackage{amssymb}
\usepackage{amsthm}
\usepackage{graphicx}
\usepackage{float}
\usepackage[dvipsnames]{xcolor}
\usepackage{hyperref}
\usepackage{url}
\usepackage{subcaption}

\hypersetup{
    colorlinks,
    citecolor=blue,
    filecolor=black,
    linkcolor=blue,
    urlcolor=black,
    linktoc=all
}

\newtheorem{proposition}{Proposition}

\theoremstyle{remark}
\newtheorem{remark}[proposition]{Remark}

\title{From Seeds to Semantics: Measuring Semantic Accessibility in Deterministic Diffusion Models
}

\author{
Kuntian Chen$^{1,*}$ \quad Wei Wei$^{1,*}$ \quad Yizhou Zeng$^{1,*}$ \quad Sophie Langer$^{2}$ \quad Mariia Seleznova$^{3}$ \quad Hung-Hsu Chou$^{1}$\\
$^{1}$Department of Mathematics, University of Pittsburgh \\
$^{2}$Mathematical Statistics, Ruhr University Bochum\\
$^{3}$Signal Theory and Statistical Learning, Technical University of Darmstadt\\
\textsuperscript{*}Equal contribution
}

\iclrfinalcopy
\begin{document}

\maketitle

\begin{abstract}
Diffusion models generate samples through a sequence of learned denoising steps, and recent work has studied how semantic structure appears along this sampling process. We study this question in deterministic samplers by measuring \emph{semantic accessibility}: how much information about a final semantic property, such as an image class label or attribute, can be extracted from the seed and intermediate states along the trajectory that produces the sample. Using DDIM sampling, for which each initial noise seed determines a unique trajectory and final image, we train separate classifiers (probes) at several points along the trajectory to predict a semantic property of the final image. We measure how well such a property can be predicted from the state at a given point using top-1 accuracy and normalized mutual information. Across MNIST, Fashion-MNIST, CIFAR-10, and CelebA, class labels and image attributes can be predicted above chance from the initial noise seed, and along DDIM trajectories, this accessibility exceeds matched-noise forward baselines. We find that semantic accessibility is substantially higher along the trajectories whose final images are classified with high confidence than along those classified with low confidence. These measurements provide a quantitative view of when semantic properties can be recovered during deterministic generation.
\end{abstract}

\section{Introduction}

Diffusion models generate data by transforming noise through a sequence of learned denoising steps~\citep{sohl2015deep,ho2020denoising,song2021scorebased}. A substantial body of work investigated when semantic structure is settled during this process, identifying a speciation transition after which the broad class structure of a sample is fixed~\citep{biroli2024dynamical}, spontaneous symmetry breaking and localized decision windows~\citep{raya2023spontaneous}, a phase transition governing recovery of high-level features under forward--backward noising~\citep{sclocchi2025phase}, and information-theoretic estimates of when class decisions are made~\citep{handke2025measuring}.

Prior work has studied semantic emergence in diffusion models mainly through distribution-level or dynamical notions. In these settings, one asks when a semantic property becomes statistically determined by the diffusion state.
These notions are informative for understanding the generative process, but they do not directly answer whether the semantic property is predictable from a single trajectory state by a fixed learned model. This distinction is especially important for deterministic models such as DDIM~\citep{song2021denoising} as opposed to probabilistic models such as DDPM~\citep{ho2020denoising}, where trajectory predictability is well-defined for the former but not the latter. For DDIM, once the initial seed, model, and sampling schedule are fixed, the final sample is already determined by the seed.
This does not mean that the semantic property is accessible from the seed, because the function from the seed to the final property may be too complex for a concrete predictor to learn. The relevant question is therefore not when the semantic property becomes fixed by the dynamics, but when it becomes predictable from the trajectory state.

We formalize this question through the notion of \emph{semantic accessibility}. A semantic property is a discrete variable associated with the final generated image, such as a class label or a binary visual attribute, assigned by a fixed classifier. At each selected point along the DDIM trajectory, we train an independent probe to predict this property from the current state. We summarize the probe output using top-1 accuracy and normalized mutual information (NMI), so that accessibility measures how much information about the final semantic property is extracted by the probe at that point.

We also study how accessibility varies across sampling trajectories. For each semantic property, we use the classifier margin on the final image as a ranking score, and partition trajectories via their confidence level. This lets us compare whether semantic properties are more accessible along trajectories whose final images are classified with higher confidence. We evaluate on MNIST, Fashion-MNIST, CIFAR-10 and CelebA~\citep{lecun1998gradient,xiao2017fashion,krizhevsky2009learning,liu2015faceattributes}.

\begin{figure}[t]
    \centering
    \includegraphics[width=0.75\textwidth]{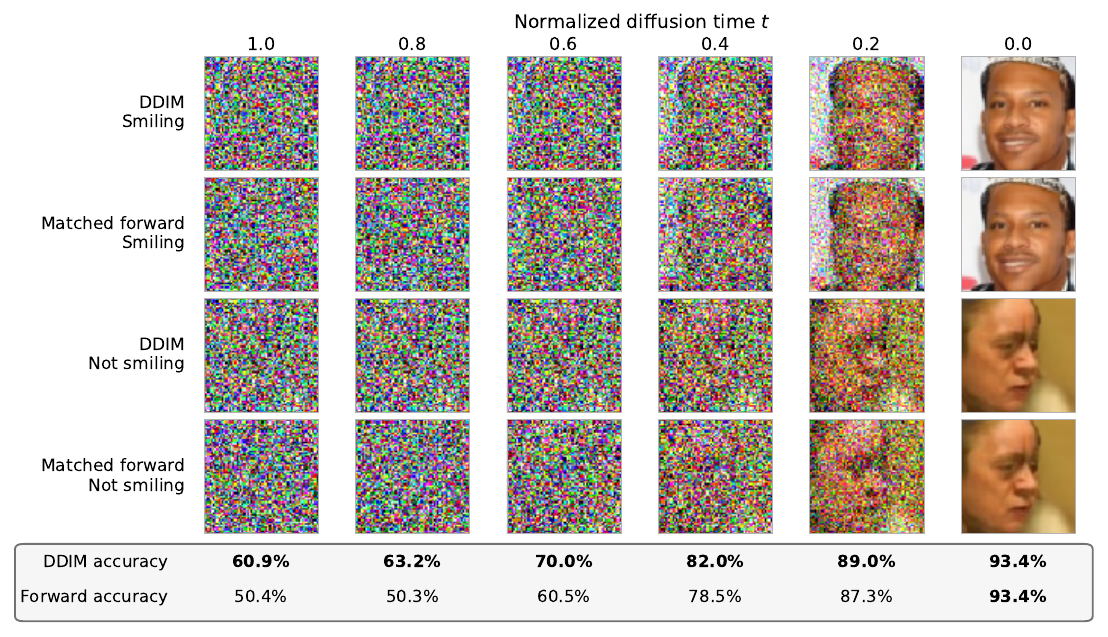}
    \caption{Representative DDIM trajectories and matched forward-corruption paths from the initial seed to the generated image. Accuracy rows report aggregate held-out performance. This illustrates the matched-noise comparison: early DDIM states can appear visually noise-like while already supporting above-chance prediction of the final semantic property. More details in Section~\ref{sec:trajectory-results}.}
    \label{fig:matched-noise-examples}
\end{figure}

\paragraph{Contributions}
\begin{enumerate}
    \item \textbf{Probe-based accessibility.} We introduce semantic accessibility as a probe-based measurement on trained deterministic diffusion samplers. It quantifies how much a final semantic property of a generated image can be extracted from states along its generation trajectory. We formalize why semantic accessibility is different from deterministic dependence: although the final property is fixed by every state along a deterministic generation trajectory, ordinary mutual information cannot distinguish whether the property is recoverable by a learned probe (Proposition~\ref{prop:mi-constant}).

    \item \textbf{Trajectory effects beyond noise level.} We show that semantic accessibility depends not only on nominal noise level, but also on how the deterministic sampler organizes individual trajectories. To separate these effects, we compare DDIM states with a matched-noise forward-corruption baseline, where clean images are corrupted to the same nominal noise levels and probed with the same architecture. DDIM states are more predictive of the final semantic property than these baselines, as shown in Figure~\ref{fig:matched-noise-examples} and Section~\ref{sec:trajectory-results}.

    \item \textbf{Confidence-stratified accessibility.} We show that semantic accessibility is not uniform across seeds. Grouping trajectories by the classifier confidence of their final images, we find that high-confidence groups have substantially higher seed-level accessibility and largely preserve this advantage along the trajectory, as shown in Figure~\ref{fig:cross-dataset-seed-confidence} and Section~\ref{sec:confidence-filtering-results}. Complementary supervised LDA--UMAP visualizations show clearer class-aligned structure among high-confidence seeds (Figure~\ref{fig:seed-lda-umap}; Appendix~\ref{app:lda-umap}).
    
    \item \textbf{Validation controls.} We validate the interpretation with randomized-target and DDPM controls. Randomized targets test whether probes can generalize arbitrary seed-label assignments, while DDPM tests whether the seed-level confidence structure persists when fresh sampling noise breaks the deterministic seed-to-sample map. These controls support the trajectory-specific interpretation of the observed accessibility (Section~\ref{sec:randomized-target-results}).
\end{enumerate}

\begin{figure}[t]
    \centering
    \includegraphics[width=0.9\textwidth]{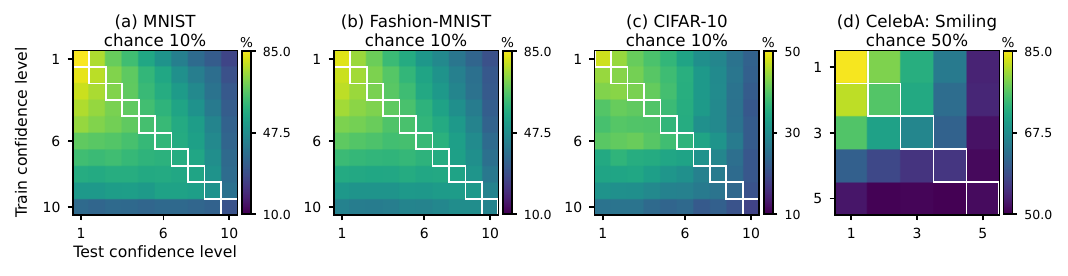}
    \caption{Confidence-stratified seed accessibility across datasets. Rows indicate the confidence group used to train the probe, columns indicate the confidence group used for held-out evaluation, and colors show mean seed-level accuracy over three probe replicates. Level 1 is highest confidence; outlined cells use matched training and evaluation groups. The stronger performance in high-confidence groups shows that semantic accessibility is not uniform across the seed space.}
    \label{fig:cross-dataset-seed-confidence}
\end{figure}

\section{Related Work}

\paragraph{Diffusion trajectories and sampling time.}
Prior work has studied diffusion time as an axis along which semantic properties become statistically determined, for example through symmetry breaking, speciation, feature recovery, or entropy-based measurements along the generative process~\citep{raya2023spontaneous,biroli2024dynamical,sclocchi2025phase,handke2025measuring}.
Timestep-dependent classifiers in classifier guidance also show that semantic prediction at noisy diffusion states can be useful during sampling~\citep{dhariwal2021diffusion}, while classifier-free guidance replaces the external classifier with conditional and unconditional score estimates~\citep{ho2022classifierfree}.
These analyses are naturally distributional rather than on the trajectory level, which motivate studying diffusion time as an axis for semantic analysis, while leaving open the trajectory-level accessibility question studied here.

\paragraph{Probing and usable information.}
Our measurement follows a common probing approach: train a restricted predictor on a representation and interpret held-out performance as information extractable by that predictor class~\citep{alain2017understanding,belinkov2022probing}. This viewpoint is formalized by predictive $\mathcal V$-information~\citep{xu2020usable}, and motivates randomized-target controls for detecting probe artifacts~\citep{hewitt2019designing}. Within diffusion models, prior work has probed denoiser activations and diffusion features across noise levels for visual representation tasks such as classification, segmentation, landmark detection, and semantic correspondence~\citep{yang2023diffusion,tang2023emergent,luo2023diffusionhyperfeatures,stracke2025cleandift}, and has also studied learned readouts, semantic latent spaces, and internal denoiser representations~\citep{luo2024readout,kwon2023diffusion,kim2025revelio}. Closest to our setting, \citet{franzese2025latent} probes noise-perturbed real images and denoiser features to study when latent factors become accessible. We instead probe states on deterministic DDIM trajectories and predict semantic properties of the generated image, with matched-noise forward baselines used for comparison.

\paragraph{Initial noise, seed effects and confidence.}
Seed selection and initial-noise optimization are known to affect rare-concept generation and prompt--image alignment, and seeds have systematic effects in text-to-image models~\citep{samuel2024generating,guo2024initno,xu2025good}. Our confidence stratification is related to selective prediction, where confidence ranks examples by reliability under varying coverage~\citep{geifman2017selective}. Because network confidences are often miscalibrated~\citep{guo2017calibration}, we use final logit margins only to rank trajectories and assign no probabilistic interpretation to their values.

\section{Framework and Metrics}

\subsection{Deterministic trajectories and time convention}
\label{sec:trajectories}

We use the continuous-time formulation, with (normalized) diffusion time $t\in[0,1]$ and variance-preserving coefficient $\bar\alpha_t$ decreasing from $\bar\alpha_0=1$ at the data to $\bar\alpha_1\approx0$ at the noise end. DDIM sampling integrates the probability-flow dynamics backward from $t=1$ to $t=0$ on a fixed grid, and we probe states at a finite set of times $\mathcal T\subset[0,1]$, with $t=1$ the initial seed and $t=0$ the generated image. We write the trajectory as
\[
X_1 \to \cdots \to X_t \to \cdots \to X_0.
\]
Once the trained denoiser and the sampling schedule are fixed, each DDIM update maps a state to the next deterministically. We write $\Phi_{t\to 0}$ for the composition of the remaining updates from time $t$, so $X_0=\Phi_{t\to 0}(X_t)$.
The exact probed positions, corresponding training timesteps, and noise-schedule values are listed in Table~\ref{tab:timesteps} of Appendix~\ref{app:experimental-details}.

\subsection{Semantic properties and probes}
\label{sec:semantic-variables-probes}

We describe the construction for one semantic property. Later in  our experiments the same construction is applied separately to each class label or attribute. Let $C:\mathcal X_0\to\mathcal Y$ be a fixed labeling rule applied to the final generated image $X_0$, and define
\[
Y=C(X_0),
\]
with $\mathcal Y$ finite. In our experiments $C$ is a classifier trained on real data and applied to the generated image, so $Y$ is a classifier-assigned value rather than a human annotation of the generated sample (Section~\ref{sec:labels-probes}).
For unstratified experiments, we train a probe $f_t$ at each probed time $t\in\mathcal T$ and write $\widehat Y_t=f_t(X_t)$. For confidence-stratified experiments, we train and evaluate probes separately within each confidence level $\ell$ (Section~\ref{sec:confidence-stratification}), and write $f_{t,\ell}$ and $\widehat Y_{t,\ell}=f_{t,\ell}(X_t)$. Probes are trained from scratch independent of $C$, and are evaluated on held-out trajectories. At $t=1$ this measures seed-level accessibility, and at $t=0$ image-level accessibility.

\subsection{Semantic Accessibility}
\label{sec:probe-extracted-information}

The following proposition formalizes why semantic accessibility is different from deterministic dependence and explains why we do not measure accessibility by the mutual information between the full trajectory state and the final semantic property.
\begin{proposition}
\label{prop:mi-constant}
Fix a trained deterministic sampler and its sampling schedule. Assume that the sampler map $\Phi_{t\to0}:\mathcal X_t\to\mathcal X_0$ and the classifier $C:\mathcal X_0\to\mathcal Y$ are measurable, with $\mathcal Y$ being finite. If $Y=C(X_0)$ and $X_0=\Phi_{t\to0}(X_t)$, then $Y$ is a deterministic measurable function of $X_t$. Consequently, the conditional entropy and mutual information, denoted as $H$ and $I$, satisfy
\[
H(Y\mid X_t)=0,
\qquad
I(X_t;Y)=H(Y).
\]
\end{proposition}
Proposition~\ref{prop:mi-constant} (proved in Appendix~\ref{app:proof-remark}) shows that full-state mutual information is not the right quantity for deterministic trajectories. Since $X_t$ determines $Y$, the value $I(X_t;Y)$ is maximal at every time and cannot distinguish the initial high-noise state from the final image. The relevant question is instead whether $Y$ can be recovered by the specified probe trained on $X_t$. We therefore measure accessibility through the probe prediction $\widehat Y_{t,\ell}$ rather than through the full state $X_t$, using held-out accuracy and the normalized mutual information (NMI)
\[
\mathrm{NMI}_{t,\ell}=\frac{I(\widehat Y_{t,\ell};Y)}{H(Y)},
\qquad
I(\widehat Y;Y)=\sum_{\hat y,y}p(\hat y,y)\log_2\frac{p(\hat y,y)}{p(\hat y)p(y)},
\]
with $p(\hat y,y)$ the empirical confusion matrix on held-out trajectories~\citep{cover2006elements}. This ratio is the asymmetric uncertainty coefficient rather than the symmetric normalized mutual information of the clustering literature. We use it because it reads directly as the fraction of label uncertainty removed by the probe prediction. For value-balanced evaluation sets, $H(Y)=\log_2 K$ for a $K$-valued semantic property, so accessibility is reported on a common NMI scale. Small empirical NMI values are compared with the finite-sample independence baseline in Appendix~\ref{app:proof-remark}.

\subsection{Confidence-based stratification}
\label{sec:confidence-stratification}

For a categorical variable the confidence score is the top-two logit margin, and for a binary attribute with scalar logit it is the absolute logit.
Each trajectory is assigned a level using only the classifier confidence of its final image. Within each semantic value, we sort trajectories by decreasing confidence, split them evenly into groups, and pool corresponding groups across values. Level 1 represents the group with the highest confidence.
We use finer partitions when enough trajectories are available, and coarser partitions when larger groups are needed for stable probe training and evaluation. The specific values and split sizes are reported in Appendix~\ref{app:experimental-details}.

Both the target and the confidence score come from the same classifier, so low-confidence trajectories are those whose final image lies near that classifier's decision boundary, where the assigned target is itself least certain. This is a property of the labeling rule rather than of the trajectory, and we return to it in Section~\ref{sec:results} when interpreting the confidence comparison.

\section{Experimental Setup}

\subsection{Datasets and Semantic Variables}
\label{sec:datasets}

We evaluate ten-class labels on MNIST, Fashion-MNIST, and CIFAR-10~\citep{lecun1998gradient,xiao2017fashion,krizhevsky2009learning}, and seven binary attributes on CelebA~\citep{liu2015faceattributes}: Smiling, Mouth Slightly Open, Wearing Lipstick, Male, Black Hair, Heavy Makeup, and High Cheekbones. CelebA attributes were selected before trajectory experiments for classifier reliability and interpretability; statistics are reported in Appendices~\ref{app:additional-metrics} and~\ref{app:celeba-analysis}.

For every semantic property and confidence level, the training, validation and test sets contain equal numbers of examples per target value, so real-data prevalence does not determine probe class balance and balanced accuracy coincides with accuracy on our evaluation sets.

\subsection{Diffusion Samplers}
\label{sec:samplers}

For MNIST, Fashion-MNIST and CIFAR-10 we train separate unconditional U-Nets on the corresponding real training sets.
For CelebA we use the unconditional CelebA-$64$ model and checkpoint released with DDIM~\citep{song2021denoising}, with the crop and resize used by that model. Architectures and training settings are in Appendix~\ref{app:experimental-details}.

All four settings use $100$-step deterministic DDIM with uniformly spaced steps.
We generate $500{,}000$ trajectories per class-label dataset and $170{,}000$ for CelebA.
For MNIST and Fashion-MNIST, we retain only the seed $X_1$ and final image $X_0$, using these datasets for endpoint comparisons. Full trajectory profiles are reported for CIFAR-10 and CelebA, where we focus our trajectory-level analysis on the more complex generation settings and, for CelebA, on multiple semantic attributes. Accordingly, claims about interior timesteps are restricted to CIFAR-10 and CelebA.
Generator quality and class/attribute counts are checked in Appendix~\ref{app:sample-quality}. FID~\citep{parmar2022cleanfid} is reported only as a sample-quality diagnostic, not as a proxy for semantic accessibility.

\subsection{Labeling Rules and Probes}
\label{sec:labels-probes}

All semantic targets are assigned by fixed classifiers. MNIST and Fashion-MNIST use the same four-layer convolutional architecture trained separately on each real training set, and CIFAR-10 uses a CIFAR-style ResNet-18~\citep{he2016deep}. On the corresponding official test splits these obtain $99.61\%$, $94.19\%$ and $93.78\%$ accuracy under deterministic test preprocessing. CelebA uses an ImageNet-pretrained ResNet-18 fine-tuned jointly on all 40 real attributes, with the seven selected logits thresholded at zero. On the official CelebA test split this classifier attains $90.65\%$ macro accuracy and $0.9313$ macro AUROC over all 40 attributes, while accuracy on the seven selected attributes ranges from $87.29\%$ to $97.98\%$.

Every probe is initialized and trained from scratch. MNIST and Fashion-MNIST probes use the same convolutional architecture as their labeling classifiers, CIFAR-10 uses a CIFAR-style ResNet-18, and each CelebA attribute uses a separate ResNet-18 with one binary output.
We use cross-entropy for the ten-class probes and unweighted binary cross-entropy for CelebA, with $3{,}000$ training examples per target value for the class-label datasets and $8{,}000$ for CelebA.
Within a dataset the same settings are used for every level, diffusion time and replicate, and checkpoints are selected on the validation split of the same confidence level.

We assess reproducibility within each fixed architecture and training protocol using three paired replicates. The held-out evaluation set is fixed across replicates, while the class-balanced training subset and the training randomness vary together.
These repeated runs assess whether the reported patterns are consistent across independently trained probes.
We report the full replicate spread in Appendix~\ref{app:probe-family} and include it in the uncertainty summaries of the main results.

We use the same architecture family for the labeling classifier and the probes. The labeling classifier must perform well on clean images, since its output defines the semantic property, and the probe architecture must also classify clean images well so that poor performance at earlier timesteps is not merely an architecture bottleneck. All probes are trained independently from the labeling classifier and are evaluated on held-out trajectories.

\subsection{Matched-Noise-Level Baseline}
\label{sec:matched-snr}

A raw accessibility value is difficult to interpret without a reference point: probe performance naturally depends on how much image signal is visible at the corresponding noise level. We therefore compare DDIM trajectory states with a model-independent matched-noise baseline. For a clean image $x_0$ with semantic variable $Y=C(x_0)$, we form
\begin{equation}
\label{eq:fwd}
x_t^{\mathrm{fwd}}=\sqrt{\bar\alpha_t}\,x_0+\sqrt{1-\bar\alpha_t}\,\varepsilon,
\qquad \varepsilon\sim\mathcal N(0,I),
\end{equation}
using the same nominal noise levels as the DDIM trajectory states, and train probes with the same architecture and training budget.

This baseline asks how much semantic information is extractable from ordinary Gaussian-corrupted clean images at the same noise scale, without using the learned reverse trajectory. Comparing DDIM trajectory probes against this baseline helps distinguish accessibility due to noise level from accessibility associated with the deterministic seed-to-sample trajectory. Section~\ref{sec:trajectory-results} reports the main Smiling comparison, and Appendix~\ref{app:additional-trajectory-results} reports the Male replication and additional trajectory results.

\subsection{Metrics and Controls}
\label{sec:metrics-controls}

We evaluate every probe on held-out trajectories and report accuracy and NMI with the held-out size $n$. All main evaluation sets are balanced across the values of $Y$, so chance accuracy is $10\%$ for the class-label experiments and $50\%$ for the binary CelebA attributes. Chance NMI is zero in the population but positive in finite samples at the level of~\eqref{eq:nmi-floor}. Unless stated otherwise we report the mean and standard deviation over three paired data-subset and training replicates, and paired bootstrap intervals over held-out trajectories for the confidence comparison. We use ``replicate'' for randomness in data subsetting and optimization, reserving ``seed'' for the initial noise vector $X_1$.

As a control task~\citep{hewitt2019designing} we randomly permute targets within each confidence level and data split, then train and evaluate probes with the same architecture and procedure. This preserves the trajectory states, class balance and splits while removing the original sample--target association, so held-out performance should remain at chance if arbitrary targets do not generalize. We summarize this control and a stochastic-sampling (DDPM) control in Section~\ref{sec:randomized-target-results}, with full results in Appendices~\ref{app:randomized-controls} and~\ref{app:ddpm}. 

\section{Results}
\label{sec:results}

We present the results in three steps. First, we compare DDIM trajectories with matched-noise forward corruption and show that endpoint semantic properties are more accessible along the deterministic trajectory than along ordinary noisy paths at the same nominal noise level (Figure~\ref{fig:matched-noise-examples} and \ref{fig:timestep-accessibility}; Section~\ref{sec:trajectory-results}). Second, we ask whether this accessibility is uniform across samples and find that final-image confidence stratifies trajectories already at the seed and throughout the reverse process (Figures~\ref{fig:cross-dataset-seed-confidence}, \ref{fig:confidence-stratified-accessibility}, and \ref{fig:seed-lda-umap}; Section~\ref{sec:confidence-filtering-results}). Third, we use randomized-target and stochastic-sampling controls to rule out two alternative explanations: arbitrary-label memorization by the probes and confidence structure that does not rely on deterministic seed-to-sample coupling (Section~\ref{sec:randomized-target-results}).

\subsection{Trajectory Accessibility and Matched-Noise Baselines}
\label{sec:trajectory-results}

\begin{figure}[t]
    \centering
    \includegraphics[width=0.74\textwidth]{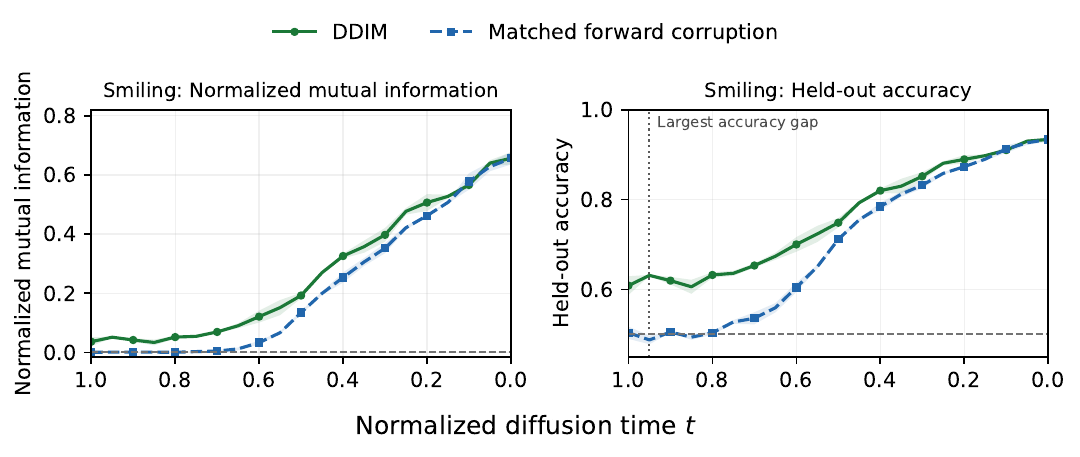}
    \caption{Semantic accessibility along deterministic DDIM trajectories compared with matched forward corruption for the CelebA attribute Smiling. Aggregate held-out NMI and accuracy over three paired probe replicates. Dashed horizontal lines mark the binary finite-sample NMI floor and $50\%$ chance accuracy. The dotted vertical line marks the largest mean accuracy gap ($14.4\%$).}
    \label{fig:timestep-accessibility}
\end{figure}

We first ask whether final semantic properties are more accessible along deterministic DDIM trajectories than along ordinary forward-corruption paths at the same nominal noise levels. Figure~\ref{fig:timestep-accessibility} compares these two paths for the CelebA attribute Smiling. The trajectory examples in Figure~\ref{fig:matched-noise-examples} illustrate the qualitative contrast: early DDIM states and matched forward-corrupted images both appear visually noise-like, but the DDIM states already yield above-chance aggregate prediction. The quantitative curves in Figure~\ref{fig:timestep-accessibility} show the same effect on the balanced held-out set.

At $t=1$, the DDIM seed yields $60.87\%$ Smiling accuracy and $0.0362$ NMI, whereas the matched forward baseline yields $50.37\%$ accuracy and $0.0006$ NMI, near the $50\%$ chance level and the finite-sample NMI floor. As the face becomes visible, forward accuracy rises and approaches the DDIM curve, but the DDIM trajectory remains more predictive at the high-noise and intermediate states.

These results support path-dependent semantic accessibility: at the same nominal noise-schedule position, the fixed probe family recovers the endpoint attribute more readily from the deterministic DDIM trajectory than from ordinary Markov corruption of that endpoint. Thus accessibility depends not only on the marginal noise level, but also on the seed-to-sample coupling induced by the sampler. The Male replication and additional matched-noise results are reported in Appendix~\ref{app:additional-trajectory-results}. This motivates asking whether the trajectory-level advantage is uniform across samples, or whether some regions of the deterministic seed-to-sample map expose the final semantic property more readily than others.

\subsection{Confidence-Stratified Accessibility}
\label{sec:confidence-filtering-results}

Having established that DDIM trajectories make endpoint semantics more accessible than matched forward corruption in Section~\ref{sec:trajectory-results}, we next ask whether this accessibility is uniform across trajectories. We group trajectories using the confidence score defined in Section~\ref{sec:confidence-stratification} and begin at the initial noise.

Figure~\ref{fig:cross-dataset-seed-confidence} repeats the initial-noise cross-confidence evaluation on MNIST, Fashion-MNIST, CIFAR-10 and CelebA Smiling. Within each dataset, accuracy is generally highest on the high-confidence test groups and decreases toward the low-confidence groups. Probes trained on higher-confidence groups also tend to transfer better across test groups. The consistent within-dataset pattern is that final labels or attributes are easier to learn and recover from initial noises associated with high-confidence final images.

We next ask whether this ordering is limited to the initial noise or continues along the reverse trajectory. Figure~\ref{fig:confidence-stratified-accessibility} follows representative high-, intermediate- and low-confidence groups on CIFAR-10 and CelebA. In all three panels, the highest-confidence group is more accessible than the lowest-confidence group at every stored state. Thus confidence stratification is not merely a final-image phenomenon; it reveals structure already present at the seed level. Table~\ref{tab:confidence-summary} summarizes the same comparison using accuracy at $t=1,0.5,0$. The CelebA ranges include all seven attributes. Complete CIFAR-10 and CelebA trajectory curves are reported in Appendix~\ref{app:additional-trajectory-results}; the MNIST and Fashion-MNIST cross-level matrices and the alternative confidence-score comparison are reported in Appendix~\ref{app:additional-confidence-results}.

\begin{figure}[t]
    \centering
    \includegraphics[width=0.86\textwidth]{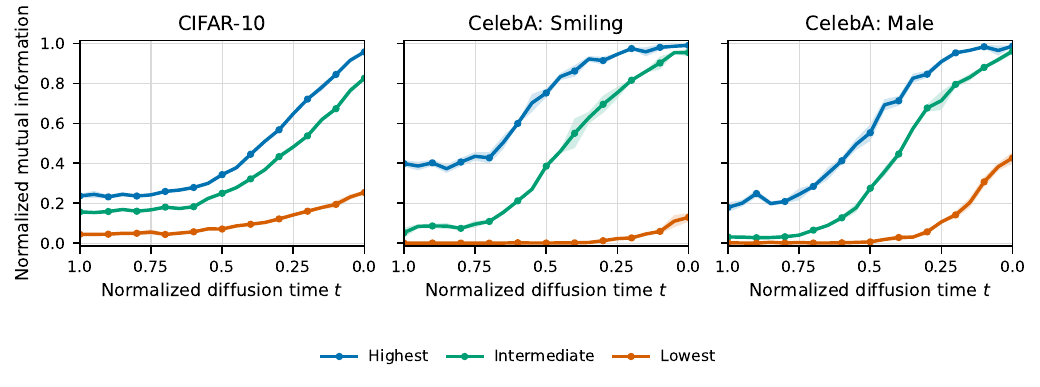}
    \caption{Confidence-stratified accessibility along deterministic DDIM trajectories from the initial seed at $t=1$ to the generated image at $t=0$. Lines show mean NMI over three probe replicates, and bands show one population standard deviation. Level 1 is the highest-confidence group; the panels show levels 1, 5 and 10 for CIFAR-10 and levels 1, 3 and 5 for CelebA.}
    \label{fig:confidence-stratified-accessibility}
\end{figure}

\begin{table}[t]
    \centering
    \small
    \setlength{\tabcolsep}{5pt}
    \caption{Confidence-stratified held-out accuracy (\%). CIFAR-10 entries are means over three probe replicates; CelebA entries give the range across the seven attribute-specific replicate means.}
    \label{tab:confidence-summary}
    \begin{tabular}{llrrrr}
        \hline
        Dataset & Confidence group & Chance & $t=1$ & $t=0.5$ & $t=0$ \\
        \hline
        CIFAR-10 & Highest (level 1)  & 10 & 47.0 & 56.9 & 98.3 \\
                 & Lowest (level 10) & 10 & 18.1 & 21.8 & 41.4 \\
        \hline
        CelebA, seven attributes & Highest (level 1) & 50 & 74.3--89.2 & 89.9--97.0 & 99.4--99.9 \\
                                 & Lowest (level 5)  & 50 & 48.9--52.2 & 51.0--56.9 & 61.3--86.4 \\
        \hline
    \end{tabular}
\end{table}

Because confidence is computed from the final image, the accuracy gap at $t=0$ is expected. The important result is that the same ordering is already present at $t=1$ and remains through the intermediate states. Final-image confidence therefore separates deterministic trajectories on which the endpoint property is consistently easier or harder for the fixed probe family to recover. This persistence across time suggests that the confidence groups correspond to different regions of the learned trajectory map, rather than only to different degrees of final-image classifier certainty.

A complementary LDA--UMAP visualization on MNIST gives the same qualitative picture: high-confidence seeds show clearer class-aligned organization than lower-confidence seeds (Figure~\ref{fig:seed-lda-umap}; see Appendix~\ref{app:lda-umap} for details and controls).

\begin{figure}[h]
    \centering
    \includegraphics[width=0.8\textwidth]{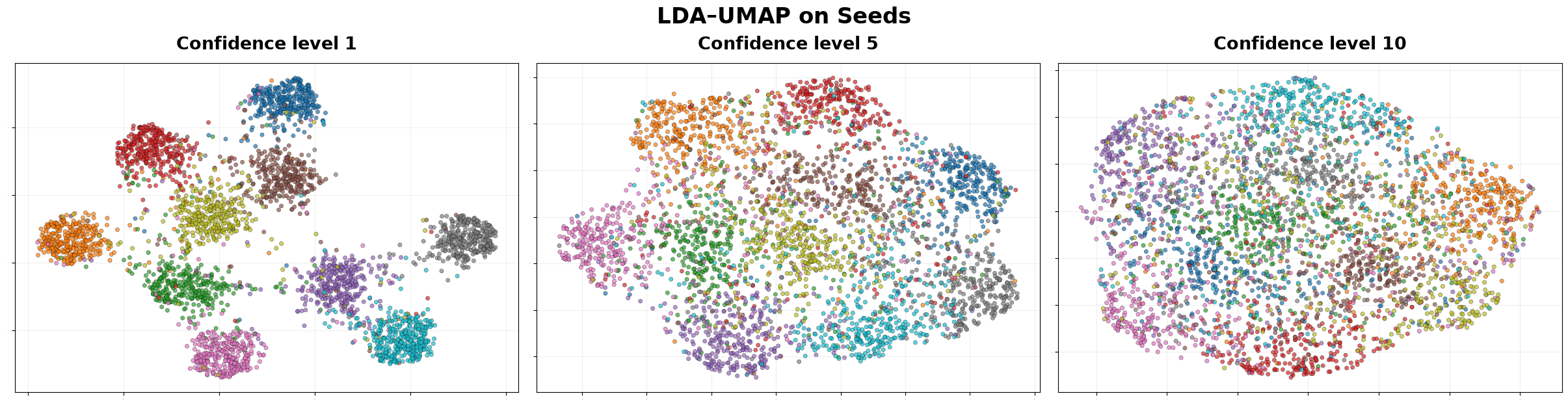}
    \caption{Supervised two-dimensional embeddings (LDA--UMAP) of MNIST seeds at confidence levels 1, 5 and 10, from highest to lowest confidence. Colors denote the final-image labels. Note that high-confidence seeds have clearer class-aligned structure than lower-confidence seeds.}
    \label{fig:seed-lda-umap}
\end{figure}

\subsection{Controls}
\label{sec:randomized-target-results}

We close with two controls that address possible alternative explanations, with full details in Appendices~\ref{app:randomized-controls} and~\ref{app:ddpm}. Under a randomized-target control on CIFAR-10, in which targets are permuted within each confidence level and split, mean held-out accuracy is $10.03\%$, essentially the $10\%$ chance level, and mean NMI is near zero. Thus, arbitrary labels do not generalize to the held-out set, indicating that the accessibility reported above is not an artifact of probe memorization. The second control tests the role of deterministic sampling. Under stochastic DDPM sampling the seed does not determine the final image, and we confirm on MNIST that seed-level accessibility is then flat across confidence levels (Appendix~\ref{app:ddpm}).

\section{Discussion and Limitations}

Many diffusion samplers are designed to produce samples from the same target distribution, but they can differ in how they couple initial noise to final images. Endpoint quality and marginal noise level do not describe this trajectory organization. Semantic accessibility provides one interpretable axis for comparing such differences. Our matched-noise and DDPM controls support this view. DDIM states are more predictive than forward-corrupted clean images at the same nominal noise level, while the seed-level confidence structure disappears when fresh sampling noise breaks the deterministic seed-to-sample relation. Thus our claim is not that deterministic samplers are universally preferable, but that accessibility distinguishes trajectory structures that endpoint distributional summaries do not capture.

This viewpoint is complementary to studies of when semantic information becomes statistically determined or visually expressed during generation. Semantic accessibility asks a different question. It asks whether a concrete predictor can recover a final semantic variable from a trajectory state. In our experiments, accessibility generally increases smoothly as the image becomes cleaner, but we do not observe a single sharp transition at which the property first becomes accessible. This is not a contradiction. A semantic variable may be determined by the dynamics, visually apparent in the image, and recoverable by a restricted probe at different points along the trajectory.

Confidence stratification shows that accessibility is not uniform over seeds. Trajectories whose final images have higher classifier confidence also have higher semantic accessibility, already at the seed and continuing through the trajectory. Since the confidence score is computed at the endpoint, this is not a method for identifying high-accessibility seeds before generation. Rather, it shows that the learned seed-to-sample map contains trajectory subsets with different accessibility profiles, suggesting that other seed- or trajectory-level rankings may be useful.

Several limitations remain. First, semantic accessibility is operational. Its value depends on the chosen semantic variable, labeling classifier, and probe family. This is part of the definition, but the measurement should be interpreted relative to the declared protocol. Second, confidence stratification is only one way to partition trajectories. Other scores may separate accessibility more effectively or apply when endpoint confidence is not well defined. Third, our conclusions are based on the trained samplers, datasets, schedules, and probes studied here. Future work should test accessibility across checkpoints, architectures, deterministic solvers, and broader probe families.

\section{Conclusion}

Deterministic generation makes a final semantic property fixed by the initial seed, but it does not make that property easy to recover from the seed or intermediate states. Semantic accessibility measures this recoverability along a deterministic diffusion trajectory.

Our experiments show that semantic accessibility is already present at the seed level, increases along the trajectory, and varies strongly across confidence-stratified seed groups. Matched-noise and DDPM controls indicate that these effects are not explained solely by nominal noise level or endpoint statistics. They reflect how the trained sampler organizes the path from noise to sample.

The broader message is that diffusion samplers can be compared not only by the distributions they produce, but also by the trajectory structures they induce. Semantic accessibility provides one way to make this comparison precise. A natural next step is to understand which training choices, sampling schedules, deterministic solvers, and probe families preserve or change this accessibility structure. 

\newpage

\section*{Acknowledgements}
M.S. acknowledges support from the Deutsche Forschungsgemeinschaft (DFG, German Research Foundation) --- Project number 572957656.

\section*{Reproducibility Statement}
Appendix~\ref{app:experimental-details} specifies dataset preprocessing, model architectures, training hyperparameters, DDIM settings, the probed time grid, sample counts and splits. Appendix~\ref{app:architectures} specifies the labeling classifiers and probes, and Appendix~\ref{app:probe-family} quantifies variation across repeated probe fits. 

\section*{AI Use Statement}
The authors used generative AI tools for writing assistance, including grammar editing, wording refinement, manuscript organization, and checking whether the exposition and logical flow were clear. The authors also used generative AI tools to search for and identify potentially relevant related work; all cited works, technical claims, and literature comparisons were checked by the authors.

Generative AI tools were also used to suggest possible additional experiments, baseline comparisons, and model-architecture choices. These suggestions were evaluated and selected by the authors. The central research idea---measuring semantic accessibility from seeds and deterministic diffusion trajectories---and the proposed methodology were developed by the authors, not generated by AI.

The authors used AI-assisted coding tools to help write and debug experimental code. The authors reviewed the code and verified that each step implemented the intended procedure. Synthetic data and generated samples used in the experiments were produced by the authors' own code and experimental pipeline, not directly generated by an external AI assistant.

The authors take full responsibility for the final content of the paper, including the research questions, methodology, experiments, code, results, interpretations, citations, and manuscript text.

\newpage

\bibliography{reference}
\bibliographystyle{iclr2027_conference}

\newpage

\appendix
\raggedbottom

\section{Proof and Remark}
\label{app:proof-remark}
\begin{proof}[Proof of Proposition \ref{prop:mi-constant}]
By assumption, $\Phi_{t\to0}$ and $C$ are measurable. Hence their composition
\[
g_t:=C\circ\Phi_{t\to0}
\]
is measurable. Since
\[
Y=C(X_0)=C(\Phi_{t\to0}(X_t))=g_t(X_t),
\]
the random variable $Y$ is $\sigma(X_t)$-measurable. Because $\mathcal Y$ is finite, the conditional distribution of $Y$ given $X_t$ is a point mass almost surely. Therefore
\[
H(Y\mid X_t)=0.
\]
Using the identity for mutual information with a finite-valued variable,
\[
I(X_t;Y)=H(Y)-H(Y\mid X_t),
\]
we obtain
\[
I(X_t;Y)=H(Y).
\]
\end{proof}

\begin{remark}
Plug-in estimates of $I(\widehat Y;Y)$ are upward biased. Under independence, $2n\hat I$ is approximately $\chi^2_{(K-1)^2}$ in nats, so
\begin{equation}
\label{eq:nmi-floor}
\mathbb E\big[\widehat{\mathrm{NMI}}\big]\approx\frac{(K-1)^2}{2n\ln 2\cdot\log_2 K}
\end{equation}
for a held-out set of size $n$~\citep{paninski2003estimation}. We report $n$ with every NMI value and compare against this floor rather than against zero. Section~\ref{sec:randomized-target-results} checks that the randomized-target control sits at it.
\end{remark}

\newpage
\section{Experimental Details}
\label{app:experimental-details}

\paragraph{Datasets and preprocessing.}
MNIST and Fashion-MNIST use the canonical $28\times28$ grayscale training sets. CIFAR-10 uses the canonical $32\times32$ RGB training set. Inputs to the three repository-trained diffusion models are mapped from $[0,1]$ to $[-1,1]$; Fashion-MNIST and CIFAR-10 additionally use random horizontal flips during diffusion training, while CIFAR-10 uses no crop augmentation for the diffusion model. CelebA uses the aligned images of size $178\times218$: following the released DDIM configuration, we crop a $128\times128$ square with top-left coordinate $(57,25)$ and resize it bilinearly to $64\times64$. The same crop and resize are used by the CelebA labeling classifier, followed by ImageNet normalization.

Labels and confidence scores are computed only from the final generated image. For MNIST and Fashion-MNIST, clamping $X_0$ to $[-1,1]$ gives exactly the representation used by the grayscale classifier after its fixed $0.5/0.5$ normalization. For CIFAR-10, we clamp $X_0$, map it to $[0,1]$, and apply the standard CIFAR-10 mean $(0.4914,0.4822,0.4465)$ and standard deviation $(0.2470,0.2435,0.2616)$. For CelebA, we apply the inverse diffusion-data transform, clamp to $[0,1]$, and then apply the ImageNet normalization above. Seed and trajectory probes receive saved raw FP32 states without input normalization or augmentation. The only exception is the MNIST/Fashion-MNIST endpoint analysis: those caches retain the clamped endpoint as uint8, which we reconstruct as $x/127.5-1$ before probe training.

\paragraph{Diffusion models.}
The MNIST, Fashion-MNIST and CIFAR-10 generators use the same residual U-Net family. Each model has base width 128, two residual blocks per resolution, dropout $0.1$, convolutional resampling and four-head self-attention at downsampling factor two. The grayscale models use channel multipliers $(1,2,2)$, while CIFAR-10 uses $(1,2,2,2)$. They were trained for 800, 800 and 1,300 epochs, respectively, with batch size 128, the noise-prediction objective, Adam at learning rate $2\times10^{-4}$, no weight decay, gradient-norm clipping at 1, and EMA decay $0.9999$. Training used BF16 autocast; all reported generations use the EMA weights and full-precision sampling. The CIFAR-10 architecture is close to the original DDIM configuration but uses four attention heads.

For CelebA we use the authors' released unconditional CelebA-$64$ DDIM checkpoint and its EMA weights~\citep{song2021denoising}. Its U-Net has base width 128, channel multipliers $(1,2,2,2,4)$, two residual blocks per resolution, dropout $0.1$ and attention at $16\times16$. The checkpoint records epoch 392 and training step 500,000. We use the released architecture and sampling convention through a compatibility implementation that strictly loads all checkpoint tensors; the project-pinned checkpoint SHA-256 begins with \mbox{e7eec1cc} and ends with \mbox{e01a39a}. The upstream release does not publish an independent checksum, so this digest identifies the file used here rather than certifying the upstream download.

All four models use $T=1{,}000$ linearly spaced diffusion variances from $10^{-4}$ to $0.02$. We sample with 100 uniformly spaced deterministic DDIM updates and $\eta=0$. CIFAR-10 and CelebA retain 21 raw FP32 states at normalized times $t\in\{1.00,0.95,\ldots,0.05,0\}$, including the initial Gaussian seed and final generated image. The full trajectory experiments evaluate all 21 states; the CIFAR-10 randomized-target control uses $t\in\{1,0.75,0.5,0.25,0\}$. MNIST and Fashion-MNIST retain the initial seed in FP32 and the clamped endpoint in uint8, and are evaluated only at these two endpoints.

\paragraph{Time convention.}
Let $k\in\{0,5,\ldots,100\}$ be the number of applied reverse updates. We report normalized diffusion time as $t=1-k/100$, so $t=1$ is the initial noise and $t=0$ is the generated image. The repository sampler used for MNIST, Fashion-MNIST and CIFAR-10 follows its historical shifted schedule: before the last update its state indices are $j_{\mathrm{repo}}=991-10k$, and the final update ends at $j_{\mathrm{repo}}=0$. The released CelebA sampler uses $j_{\mathrm{off}}=990-10k$ before the last update and the clean-endpoint convention $j_{\mathrm{off}}=-1$, for which $\bar\alpha_{-1}=1$. Table~\ref{tab:timesteps} reports the exact mapping and the official-schedule coefficients used by the matched forward-corruption experiment. Both schedules use
\begin{equation}
    \bar\alpha_j=\prod_{s=0}^{j}(1-\beta_s),
    \qquad \beta_s=10^{-4}+\frac{s}{999}(0.02-10^{-4}).
\end{equation}
The one-index difference between schedules is small away from the endpoint but is recorded explicitly for reproducibility rather than treating the conventions as identical.

\begin{table}[H]
\centering
\scriptsize
\caption{The 21 stored and probed trajectory positions. $t$ is normalized diffusion time and $k$ is the corresponding number of applied reverse updates. $j_{\mathrm{repo}}$ is used by the MNIST, Fashion-MNIST and CIFAR-10 sampler; $j_{\mathrm{off}}$ is the released CelebA convention. Coefficients are shown for $j_{\mathrm{off}}$. At the repository endpoint, $j_{\mathrm{repo}}=0$, $\bar\alpha=0.9999$ and SNR $=9{,}999$; the official clean endpoint has $\bar\alpha_{-1}=1$ and infinite nominal SNR.}
\label{tab:timesteps}
\begin{tabular}{rrrrrrr}
\hline
$t$ & $k$ & $j_{\mathrm{repo}}$ & $j_{\mathrm{off}}$ & $\bar\alpha_j$ & $\sqrt{\bar\alpha_j}$ & $\bar\alpha_j/(1-\bar\alpha_j)$ \\
\hline
1.00 & 0   & 991 & 990 & $4.84\!\times\!10^{-5}$ & $6.95\!\times\!10^{-3}$ & $4.84\!\times\!10^{-5}$ \\
0.95 & 5   & 941 & 940 & $1.28\!\times\!10^{-4}$ & $1.13\!\times\!10^{-2}$ & $1.28\!\times\!10^{-4}$ \\
0.90 & 10  & 891 & 890 & $3.24\!\times\!10^{-4}$ & $1.80\!\times\!10^{-2}$ & $3.24\!\times\!10^{-4}$ \\
0.85 & 15  & 841 & 840 & $7.77\!\times\!10^{-4}$ & $2.79\!\times\!10^{-2}$ & $7.77\!\times\!10^{-4}$ \\
0.80 & 20  & 791 & 790 & $1.77\!\times\!10^{-3}$ & $4.21\!\times\!10^{-2}$ & $1.77\!\times\!10^{-3}$ \\
0.75 & 25  & 741 & 740 & $3.84\!\times\!10^{-3}$ & $6.19\!\times\!10^{-2}$ & $3.85\!\times\!10^{-3}$ \\
0.70 & 30  & 691 & 690 & $7.90\!\times\!10^{-3}$ & $8.89\!\times\!10^{-2}$ & $7.97\!\times\!10^{-3}$ \\
0.65 & 35  & 641 & 640 & $1.55\!\times\!10^{-2}$ & $0.124$ & $1.57\!\times\!10^{-2}$ \\
0.60 & 40  & 591 & 590 & $2.88\!\times\!10^{-2}$ & $0.170$ & $2.97\!\times\!10^{-2}$ \\
0.55 & 45  & 541 & 540 & $5.11\!\times\!10^{-2}$ & $0.226$ & $5.38\!\times\!10^{-2}$ \\
0.50 & 50  & 491 & 490 & $8.60\!\times\!10^{-2}$ & $0.293$ & $9.41\!\times\!10^{-2}$ \\
0.45 & 55  & 441 & 440 & $0.138$ & $0.371$ & $0.160$ \\
0.40 & 60  & 391 & 390 & $0.210$ & $0.458$ & $0.265$ \\
0.35 & 65  & 341 & 340 & $0.304$ & $0.551$ & $0.436$ \\
0.30 & 70  & 291 & 290 & $0.418$ & $0.647$ & $0.719$ \\
0.25 & 75  & 241 & 240 & $0.548$ & $0.740$ & $1.21$ \\
0.20 & 80  & 191 & 190 & $0.683$ & $0.827$ & $2.16$ \\
0.15 & 85  & 141 & 140 & $0.810$ & $0.900$ & $4.26$ \\
0.10 & 90  & 91  & 90  & $0.913$ & $0.956$ & $10.5$ \\
0.05 & 95  & 41  & 40  & $0.980$ & $0.990$ & $48.4$ \\
0.00 & 100 & 0   & $-1$ & $1$ & $1$ & $\infty$ \\
\hline
\end{tabular}
\end{table}

\paragraph{Generated populations and splits.}
Table~\ref{tab:population-splits} summarizes the main fixed allocations. For the categorical datasets, generation is first balanced to the smallest predicted class and confidence levels are then formed separately within each class. For each CelebA attribute, predicted positives and negatives are ranked separately by $|\mathrm{logit}|$, divided into five levels, and then pooled; different attributes therefore induce different level memberships on the same 170,000 generated images. Validation and test membership is fixed across all trajectory states and replicates. The three training replicates use independently permuted class-balanced subsets while leaving those validation and test sets unchanged.

\begin{table}[H]
\centering
\small
\caption{Generated populations and main probe splits. Counts in the train, validation and test columns are per semantic value and confidence level. MNIST and Fashion-MNIST use the same fixed splits for their seed and final-image endpoint probes.}
\label{tab:population-splits}
\begin{tabular}{lrrrrrr}
\hline
Dataset & Generated & Levels & Balanced population & Train & Validation & Test \\
\hline
MNIST & $500{,}000$ & 10 & $45{,}443$/class & $3{,}000$ & $454$ & $454$ \\
Fashion-MNIST & $500{,}000$ & 10 & $45{,}626$/class & $3{,}000$ & $456$ & $456$ \\
CIFAR-10 & $500{,}000$ & 10 & $41{,}325$/class & $3{,}000$ & $413$ & $413$ \\
CelebA & $170{,}000$ & 5/attribute & attribute-specific & $8{,}000$ & $250$ & $500$ \\
\hline
\end{tabular}
\end{table}

\paragraph{Compute and artifact reuse.}
Trajectory generation is performed once per dataset and stored in row-aligned shards; every probe replicate reads the same cached states rather than rerunning the sampler. The complete MNIST/Fashion-MNIST seed and endpoint cross-confidence grids contain 120 trained probes and record 1.96 summed GPU-hours of probe training. The CIFAR-10 trajectory grid contains 630 probes and records 24.54 summed GPU-hours. The seven-attribute CelebA grid contains 2,205 probes and records 60.58 summed GPU-hours. These totals sum per-cell training time and exclude diffusion-model training, labeling-classifier training, learning-curve pilots and the separate controls. Each cell stores its protocol, selected-epoch metrics, held-out logits, predictions and sample identifiers; the complete aggregates independently reconstruct fixed membership and recompute the reported metrics from these saved outputs.

\subsection{Qualitative Generation Checks}
\label{app:sample-quality}

Generation-quality checks confirm the generators are not degenerate. The CIFAR-10 model obtains a Clean-FID of $5.65$ from $50{,}000$ samples against CIFAR-10 training statistics. Among the $500{,}000$ MNIST and Fashion-MNIST samples, every predicted class contains at least $45{,}443$ and $45{,}626$ examples respectively before balancing. For CelebA, every selected attribute has enough examples of each predicted value in every level for the fixed $8{,}000/250/500$ split.

Figure~\ref{fig:unconditional-sample-grid} shows an unfiltered qualitative check of all four generators. For each complete generated population we draw eight sample identifiers uniformly without replacement using one fixed random-number stream (seed 20260916). Selection does not read the generated label, confidence score or confidence level. The displayed MNIST and Fashion-MNIST images are the cached uint8 endpoints divided by 255; CIFAR-10 and CelebA use $\operatorname{clip}((X_0+1)/2,0,1)$ for display.

\begin{figure}[H]
    \centering
    \includegraphics[width=\textwidth]{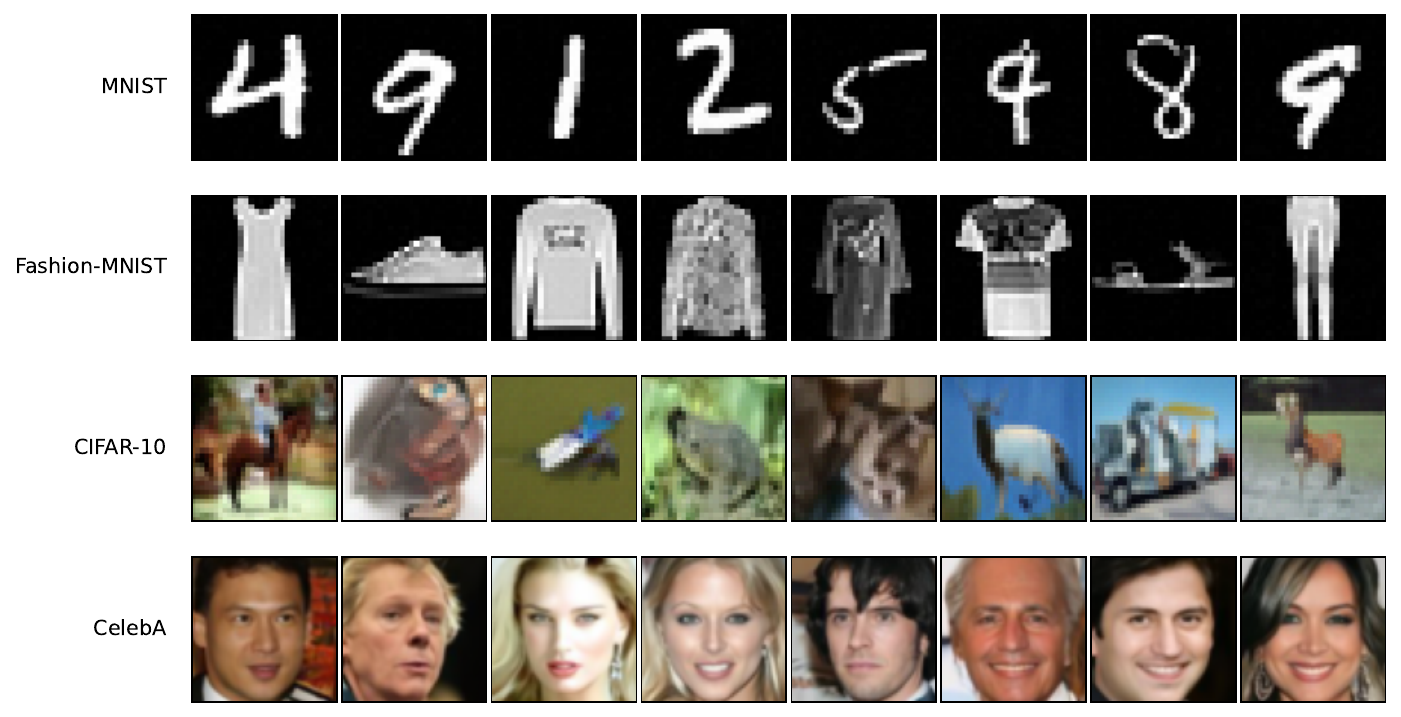}
    \caption{Unfiltered unconditional samples from the four generators. Each row contains eight samples drawn uniformly without replacement from the full generated population using the fixed rule above; no confidence filtering or manual selection is applied.}
    \label{fig:unconditional-sample-grid}
\end{figure}

\newpage
\section{Classifier and Probe Architectures}
\label{app:architectures}

\paragraph{Labeling classifiers.}
MNIST and Fashion-MNIST use the same 871,018-parameter CNN topology: two $3\times3$ convolution--batch-normalization--ReLU blocks with 32 channels, max pooling and dropout, followed by the corresponding pair with 64 channels, a 256-unit fully connected layer and a ten-way output. The model is trained on a fixed $55{,}000/5{,}000$ split of the canonical training set for 50 epochs with cross-entropy, Adam at learning rate $10^{-3}$, and cosine learning-rate decay; the historical checkpoint uses the final epoch. Fashion-MNIST uses random horizontal flips during training, but all reported held-out evaluation uses the deterministic canonical transform. The fixed checkpoints obtain $99.61\%$ and $94.19\%$ accuracy on the official MNIST and Fashion-MNIST test sets.

The CIFAR-10 labeling classifier is a ResNet-18 trained from scratch, with its ImageNet stem replaced by a $3\times3$, stride-one convolution and no initial max pooling. It is trained for 200 epochs on all 50,000 training images using random $32\times32$ crops with four-pixel padding, horizontal flips, cross-entropy, Adam at learning rate $10^{-3}$ and weight decay $5\times10^{-4}$, cosine decay, batch size 256 and mixed precision. Its official-test accuracy is $93.78\%$. The historical protocol selected the checkpoint by official-test accuracy; consequently this test set is not a selection-independent evaluation set for the labeling classifier. Probe checkpoints, by contrast, are always selected on separate generated validation trajectories.

The CelebA labeling classifier is an ImageNet-1K-V1 initialized ResNet-18 with one joint 40-logit output layer. It is fine-tuned on the official $162{,}770/19{,}867/19{,}962$ train/validation/test split with unweighted binary cross-entropy over all attributes, AdamW at learning rate $10^{-4}$ and weight decay $10^{-4}$, batch size 128, and at most 20 epochs. A ReduceLROnPlateau schedule halves the learning rate on validation plateaus, early stopping has patience five, and checkpoint selection uses validation macro balanced accuracy. On the official test split the fixed classifier reaches macro raw accuracy $90.65\%$, macro balanced accuracy $80.42\%$ and macro AUROC $0.9313$. The seven per-attribute results used in the paper are reported in Table~\ref{tab:celeba-attrs}.

For the categorical datasets, the semantic target is the top-logit class and confidence is the top-one minus top-two logit margin. For a CelebA attribute, the target is the sign of the corresponding logit and confidence is its absolute value. Confidence levels are always constructed within the assigned semantic value, so level membership cannot be explained by class prevalence alone.

\paragraph{Probe architectures and inputs.}
The MNIST and Fashion-MNIST probes use exactly the grayscale CNN topology above, but every probe is initialized from scratch. CIFAR-10 probes use the same CIFAR-stem ResNet-18 architecture as its labeling classifier, also from scratch. CelebA probes use a standard torchvision ResNet-18 initialized from scratch with a single binary logit. Thus the labeler and probe are architecture-matched within a dataset but share no weights, features or logits. Each probe sees only a raw trajectory state and its final-image target. No remaining DDIM updates, final image, labeling-classifier output or confidence score is provided as an input.

\begin{table}[H]
\centering
\scriptsize
\caption{Training protocol for the main probes. All models use constant learning rate, CUDA mixed precision and no input augmentation. Checkpoints are selected using only the fixed validation set from the training confidence level.}
\label{tab:probe-training}
\begin{tabular}{lrrrrrr}
\hline
Dataset & Loss & Optimizer & LR & Weight decay & Max epochs & Patience \\
\hline
MNIST & cross-entropy & AdamW & $3\times10^{-4}$ & $10^{-4}$ & 100 & 30 \\
Fashion-MNIST & cross-entropy & AdamW & $3\times10^{-4}$ & $10^{-4}$ & 100 & 30 \\
CIFAR-10 & cross-entropy & AdamW & $10^{-3}$ & $10^{-4}$ & 100 & 15 \\
CelebA & binary cross-entropy & AdamW & $3\times10^{-4}$ & $10^{-4}$ & 100 & 30 \\
\hline
\end{tabular}
\end{table}

All probe runs use batch size 128 and evaluation batch size 512. Categorical checkpoint selection maximizes validation accuracy; CelebA selection maximizes validation balanced accuracy, which equals validation accuracy because each split is exactly balanced. When multiple epochs attain the same best value, the first is retained. Test outputs are produced only after reloading that checkpoint.

\paragraph{Replicates and independence.}
For each main condition we train three paired replicates. A replicate seed jointly fixes the class-wise training subset, parameter initialization, minibatch order and dropout randomness, and the same replicate index is reused across confidence levels and trajectory states. The resulting spread therefore measures the combined finite-subset and training variation; it should not be interpreted as isolating initialization alone. Validation and test examples are common across replicates. The labeling classifier is frozen before any generated population, target, confidence level or probe split is constructed, and probe training never updates or calls it. The randomized-target control retains this complete training and selection protocol while replacing the targets separately within each fixed data split.

\newpage
\section{Variation across Repeated Probe Fits}
\label{app:probe-family}
Accessibility is estimated with a trained probe, so we repeat every main condition three times. As detailed in Appendix~\ref{app:architectures}, a replicate changes both the class-balanced training subset and the training randomness (parameter initialization, minibatch order and dropout), while keeping the held-out validation and test sets fixed. The variation therefore combines finite-subset and optimization effects.

Figures~\ref{fig:full-cifar-trajectories}--\ref{fig:full-celeba-nmi-trajectories} show the mean trajectories and population-standard-deviation bands over these three paired replicates. Across the $210$ CIFAR-10 confidence-level-by-state conditions, the median across-replicate standard deviation is $0.75$ accuracy points and $0.0079$ NMI; the respective maxima are $2.09$ points and $0.0267$. Across the $735$ CelebA attribute-by-level-by-state conditions, the medians are $0.78$ accuracy points and $0.0122$ NMI, and the maxima are $3.51$ points and $0.0740$. The bands are generally narrow relative to the confidence-level and trajectory differences described in the main text. This check shows that those reported patterns are not produced by a single training subset or fitted probe.

\newpage
\section{Additional Metrics and Attribute Statistics}
\label{app:additional-metrics}

\paragraph{Accuracy and NMI.}
Figure~\ref{fig:accuracy-nmi} compares the two held-out accessibility metrics over every completed evaluation condition. Each point is the mean over three paired replicates. The MNIST and Fashion-MNIST panels contain all $10\times10$ train-level-by-test-level cells from the seed-level cross-confidence experiment; the CIFAR-10 panel contains all $10\times21$ confidence-level-by-trajectory-state cells; and the CelebA panel contains all $7\times5\times21$ attribute-by-level-by-state cells. Accuracy and NMI are strongly associated within every dataset, with Pearson correlations of $0.984$, $0.991$, $0.988$ and $0.949$, respectively. The relationship is not linear, especially for the balanced binary CelebA tasks, so we retain both metrics rather than treating one as a rescaling of the other.

\begin{figure}[H]
    \centering
    \includegraphics[width=\textwidth]{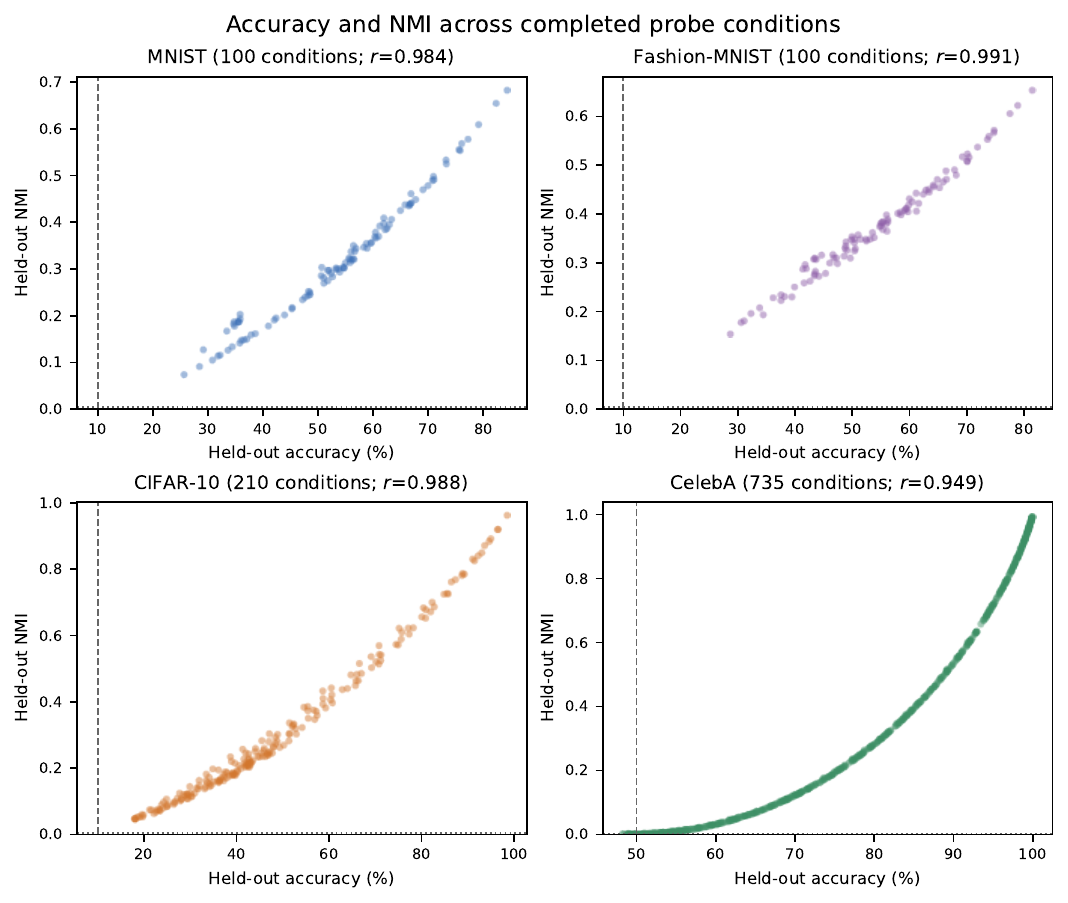}
    \caption{Held-out NMI against held-out accuracy across completed probe conditions. Points are condition means over three paired replicates. Vertical dashed lines mark chance accuracy and horizontal dotted lines mark the finite-sample NMI floor from Eq.~\eqref{eq:nmi-floor}.}
    \label{fig:accuracy-nmi}
\end{figure}

\paragraph{Held-out sizes and replicate variation.}
Table~\ref{tab:evaluation-scale} reports the fixed held-out size and NMI independence floor for each dataset. For completeness, it also summarizes uncertainty over the three paired replicates: for every condition we form a two-sided $95\%$ Student-$t$ interval for the mean and report the median half-width across conditions. These replicate intervals describe combined subset and probe-training variation and are distinct from the trajectory-wise paired bootstrap intervals in Figure~\ref{fig:confidence-stratified-accessibility}. With only three replicates they should be read descriptively. The unusually large maximum MNIST half-width ($38.42$ accuracy points and $0.249$ NMI) comes from the retained low-confidence optimization failure; the median remains representative of the rest of the grid.

\begin{table}[H]
    \centering
    \small
    \caption{Evaluation scale, finite-sample NMI floor and typical replicate uncertainty. Accuracy interval half-widths are in percentage points. ``Conditions'' excludes the three paired replicates.}
    \label{tab:evaluation-scale}
    \begin{tabular}{lrrrrrr}
        \hline
        Dataset & $K$ & $n$ & NMI floor & Conditions & Acc. half-width & NMI half-width \\
        \hline
        MNIST & 10 & $4{,}540$ & $0.00387$ & 100 & $1.63$ & $0.0149$ \\
        Fashion-MNIST & 10 & $4{,}560$ & $0.00386$ & 100 & $1.61$ & $0.0136$ \\
        CIFAR-10 & 10 & $4{,}130$ & $0.00426$ & 210 & $2.29$ & $0.0241$ \\
        CelebA & 2 & $1{,}000$ & $0.000721$ & 735 & $2.36$ & $0.0370$ \\
        \hline
    \end{tabular}
\end{table}

\paragraph{CelebA labeling-classifier statistics.}
Table~\ref{tab:celeba-attrs} reports performance of the fixed ImageNet-initialized ResNet-18 labeling classifier on the official CelebA test split ($n=19{,}962$). Prevalence is computed from the human attribute annotations and the accuracy and AUROC are evaluated against those annotations. The trajectory experiments subsequently balance both predicted attribute values within every confidence level, so these prevalences do not determine probe class balance.

\begin{table}[H]
    \centering
    \caption{Real-test statistics for the seven CelebA attributes used as semantic properties. Accuracy is raw binary accuracy.}
    \label{tab:celeba-attrs}
    \begin{tabular}{lrrr}
        \hline
        Attribute & Positive prevalence & Accuracy & AUROC \\
        \hline
        Smiling & $50.03\%$ & $92.33\%$ & $0.9784$ \\
        Mouth Slightly Open & $49.51\%$ & $93.12\%$ & $0.9805$ \\
        Wearing Lipstick & $52.19\%$ & $93.66\%$ & $0.9871$ \\
        Male & $38.65\%$ & $97.98\%$ & $0.9978$ \\
        Black Hair & $27.16\%$ & $88.75\%$ & $0.9516$ \\
        Heavy Makeup & $40.50\%$ & $91.16\%$ & $0.9747$ \\
        High Cheekbones & $48.18\%$ & $87.29\%$ & $0.9449$ \\
        \hline
    \end{tabular}
\end{table}

\newpage
\section{Additional Trajectory Results}
\label{app:additional-trajectory-results}

\paragraph{MNIST and Fashion-MNIST endpoints.}
Table~\ref{tab:mnist-fashion-endpoints} compares probes trained from scratch on the initial seed with independently trained probes on the final image. The architecture, fixed confidence-level splits, $3{,}000$ examples per class, three paired replicates and optimization protocol are identical at the two endpoints. These are probe results, not the accuracy of the labeling classifier. Averaging accuracy over the ten equally sized confidence levels within each replicate, the final-image probes obtain $99.35\%\pm0.02\%$ on MNIST and $96.19\%\pm0.02\%$ on Fashion-MNIST, compared with $60.88\%\pm0.96\%$ and $62.01\%\pm0.20\%$ at the seed. Endpoint accuracy remains confidence dependent: it falls to $93.75\%\pm0.23\%$ for MNIST level 10 and $76.38\%\pm0.35\%$ for Fashion-MNIST level 10, so the seed-level confidence hierarchy should be read relative to different final-image ceilings.

\begin{table}[H]
\centering
\scriptsize
\caption{Matched seed and final-image endpoint probes on MNIST and Fashion-MNIST. Entries are mean $\pm$ population standard deviation over three paired replicates. Accuracy is in percent; NMI is unitless. Each level contains $n=4{,}540$ held-out MNIST trajectories or $n=4{,}560$ Fashion-MNIST trajectories. Level 1 is highest confidence.}
\label{tab:mnist-fashion-endpoints}
\begin{tabular}{lrrrrr}
\hline
Dataset & Level & Seed accuracy & Image accuracy & Seed NMI & Image NMI \\
\hline
MNIST & 1  & $84.19\pm0.44$ & $100.00\pm0.00$ & $0.678\pm0.007$ & $1.000\pm0.000$ \\
      & 2  & $77.12\pm0.28$ & $99.98\pm0.02$  & $0.574\pm0.003$ & $0.999\pm0.001$ \\
      & 3  & $70.70\pm0.41$ & $99.99\pm0.01$  & $0.486\pm0.005$ & $1.000\pm0.000$ \\
      & 4  & $66.82\pm0.39$ & $100.00\pm0.00$ & $0.438\pm0.004$ & $1.000\pm0.000$ \\
      & 5  & $62.88\pm0.76$ & $99.99\pm0.02$  & $0.391\pm0.008$ & $1.000\pm0.001$ \\
      & 6  & $59.60\pm0.87$ & $99.99\pm0.01$  & $0.354\pm0.008$ & $1.000\pm0.000$ \\
      & 7  & $56.29\pm1.09$ & $99.96\pm0.00$  & $0.319\pm0.007$ & $0.999\pm0.000$ \\
      & 8  & $53.69\pm0.76$ & $99.95\pm0.01$  & $0.297\pm0.008$ & $0.998\pm0.000$ \\
      & 9  & $48.24\pm0.31$ & $99.94\pm0.02$  & $0.249\pm0.006$ & $0.998\pm0.001$ \\
      & 10 & $29.21\pm9.16$ & $93.75\pm0.23$  & $0.125\pm0.042$ & $0.857\pm0.005$ \\
\hline
Fashion-MNIST & 1  & $81.40\pm0.48$ & $99.99\pm0.01$ & $0.649\pm0.007$ & $1.000\pm0.000$ \\
      & 2  & $74.74\pm0.43$ & $99.99\pm0.02$ & $0.563\pm0.005$ & $1.000\pm0.001$ \\
      & 3  & $70.04\pm0.78$ & $99.93\pm0.03$ & $0.504\pm0.010$ & $0.998\pm0.001$ \\
      & 4  & $66.43\pm0.81$ & $99.74\pm0.00$ & $0.467\pm0.009$ & $0.993\pm0.000$ \\
      & 5  & $63.33\pm0.62$ & $99.79\pm0.05$ & $0.447\pm0.007$ & $0.994\pm0.001$ \\
      & 6  & $60.46\pm0.65$ & $99.38\pm0.09$ & $0.422\pm0.008$ & $0.985\pm0.002$ \\
      & 7  & $56.21\pm1.07$ & $98.71\pm0.05$ & $0.381\pm0.006$ & $0.972\pm0.001$ \\
      & 8  & $53.52\pm0.27$ & $96.73\pm0.11$ & $0.359\pm0.003$ & $0.938\pm0.002$ \\
      & 9  & $50.37\pm0.35$ & $91.22\pm0.17$ & $0.338\pm0.005$ & $0.863\pm0.002$ \\
      & 10 & $43.58\pm0.35$ & $76.38\pm0.35$ & $0.280\pm0.004$ & $0.684\pm0.003$ \\
\hline
\end{tabular}
\end{table}

\paragraph{Separate forward-corruption realizations and replication.}
Figure~\ref{fig:matched-noise-smiling-male} shows the two independent Markov forward-corruption realizations separately for both attributes. The Smiling realizations, whose pointwise average appears in Figure~\ref{fig:timestep-accessibility}, agree closely. Male was specified as a replication attribute before the experiment. At $t=1$, its DDIM seed yields $58.93\%$ accuracy and $0.0233$ NMI, compared with $51.07\%/51.57\%$ accuracy and $0.0005/0.0010$ NMI for the two forward-corruption realizations. Across the 20 non-endpoint positions, forward accuracy is lower than reverse accuracy at $16/20$ and $17/20$ positions, with mean reverse advantages of $3.10$ and $3.29$ percentage points. The effect is smaller than for Smiling but has the same direction. At the shared $t=0$ endpoint all three Male arms coincide at $96.23\%$ accuracy and $0.7715$ NMI.

\begin{figure}[H]
    \centering
    \includegraphics[width=\textwidth]{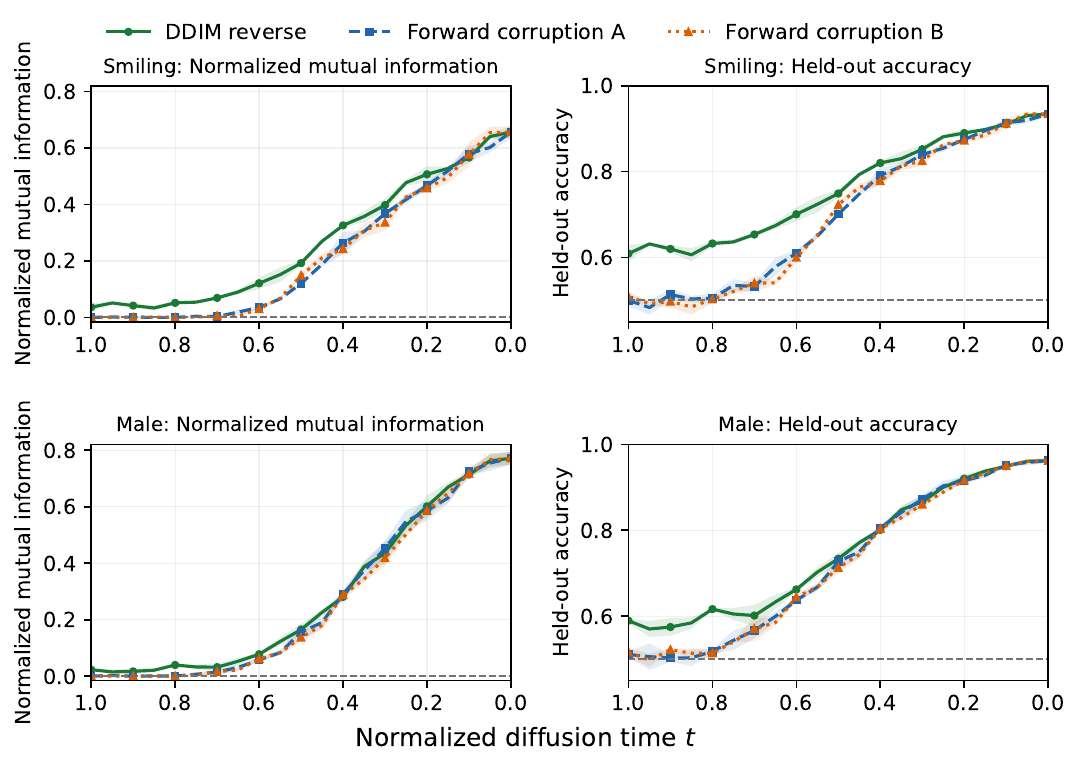}
    \caption{Generated-reverse versus generated-forward matched-noise comparisons for Smiling (top) and Male (bottom), with the two independent forward-corruption realizations shown separately. Plotting conventions and the balanced held-out size ($n=1{,}000$ per attribute) match Figure~\ref{fig:timestep-accessibility}.}
    \label{fig:matched-noise-smiling-male}
\end{figure}

\paragraph{CIFAR-10 trajectories.}
Figure~\ref{fig:full-cifar-trajectories} reports both held-out metrics at every stored trajectory position and confidence level. The full curves show the same confidence hierarchy as the representative levels in Figure~\ref{fig:confidence-stratified-accessibility}: higher-confidence trajectories are generally more accessible throughout the reverse process. They also expose the confidence-dependent endpoint ceiling, so differences at earlier states combine when a variable becomes usable by the probe with how reproducibly the final labeling rule can be recovered from each stratum.

\begin{figure}[H]
    \centering
    \includegraphics[width=\textwidth]{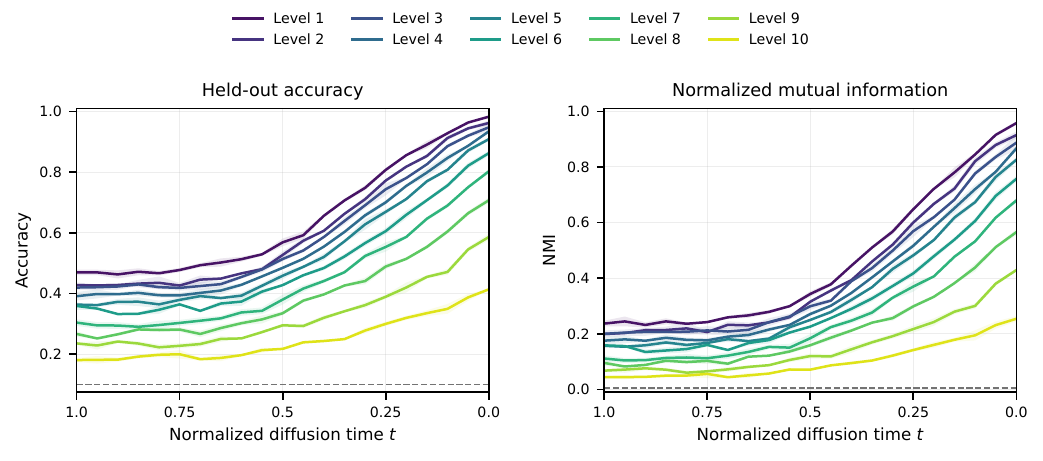}
    \caption{Full CIFAR-10 trajectory curves from $t=1$ to $t=0$ for all ten confidence levels. Lines show means over three paired data-subset and training replicates, and bands show the population standard deviation across those replicates. Each evaluation uses $n=4{,}130$ balanced held-out trajectories. Dashed horizontal lines mark $10\%$ chance accuracy and the finite-sample NMI floor from Eq.~(\ref{eq:nmi-floor}). Level 1 is the highest-confidence stratum.}
    \label{fig:full-cifar-trajectories}
\end{figure}

\paragraph{CelebA trajectories.}
Figures~\ref{fig:full-celeba-accuracy-trajectories} and~\ref{fig:full-celeba-nmi-trajectories} report the corresponding curves for all seven attributes and all five confidence levels. Accuracy and NMI give the same qualitative ordering within each attribute, while the shape and endpoint ceiling vary across attributes. These line plots are the complete confidence-stratified trajectory result. Appendix~\ref{app:additional-confidence-results} instead reports the distinct seed-level cross-confidence transfer analysis and a sensitivity check using an alternative confidence score. Bands in these figures and Figure~\ref{fig:confidence-stratified-accessibility} show the population standard deviation across the three paired probe replicates.

\begin{figure}[H]
    \centering
    \includegraphics[width=\textwidth]{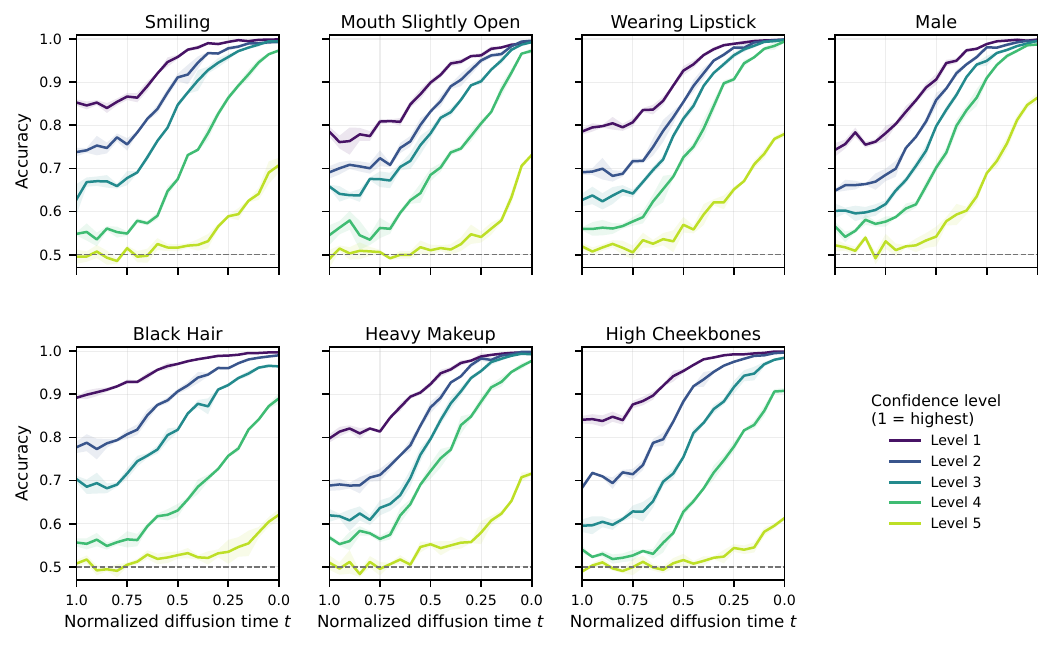}
    \caption{Full CelebA held-out accuracy curves for all seven attributes and five confidence levels. Lines show means over three paired data-subset and training replicates, and bands show the population standard deviation across those replicates. Every cell uses $n=1{,}000$ balanced held-out trajectories; dashed lines mark $50\%$ chance accuracy.}
    \label{fig:full-celeba-accuracy-trajectories}
\end{figure}

\begin{figure}[H]
    \centering
    \includegraphics[width=\textwidth]{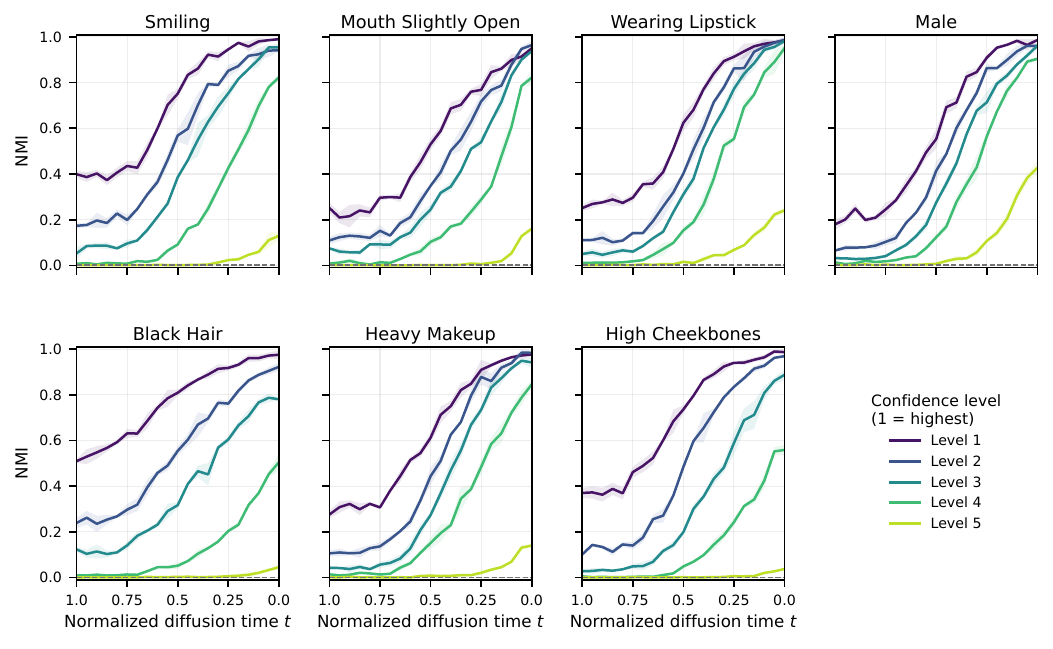}
    \caption{Full CelebA held-out NMI curves under the same protocol as Figure~\ref{fig:full-celeba-accuracy-trajectories}. Dashed lines mark the binary finite-sample NMI floor from Eq.~(\ref{eq:nmi-floor}).}
    \label{fig:full-celeba-nmi-trajectories}
\end{figure}

\newpage
\section{Additional Confidence-Stratification Results}
\label{app:additional-confidence-results}

\paragraph{Cross-level transfer at the seed.}
The within-level experiments train and test on the same confidence stratum, so their hierarchy could reflect either a shared seed-to-label relation of different strength or level-specific predictive structure. We distinguish these possibilities on MNIST and Fashion-MNIST by training a separate initial-seed probe on each confidence level and evaluating every trained probe on all ten fixed test levels. The architecture and optimization protocol are unchanged, each training level contributes $3{,}000$ examples per class, and the held-out size is $n=4{,}540$ per MNIST test level and $n=4{,}560$ per Fashion-MNIST test level. Figure~\ref{fig:cross-level-confidence-transfer} averages the three paired replicates for each train--test pair.

The diagonal reproduces the same-level hierarchy: MNIST accuracy falls from $84.19\%$ at level 1 to $29.21\%$ at level 10, and Fashion-MNIST from $81.40\%$ to $43.58\%$. Off-diagonal accuracy remains above the balanced $10\%$ chance level in every cell, with minima of $25.72\%$ and $28.74\%$, so part of the seed--label relation transfers across confidence strata. Test-set confidence nevertheless remains a strong constraint: averaging over all training levels, accuracy drops from $66.61\%$ to $32.49\%$ between MNIST test levels 1 and 10, and from $65.93\%$ to $35.74\%$ on Fashion-MNIST. The matrices are also asymmetric, with mean absolute differences between opposite train--test directions of $7.24$ and $7.28$ percentage points. Thus the hierarchy is not explained by a single level-invariant classifier whose score is merely attenuated on lower-confidence examples; both the training distribution and the test stratum matter. All predeclared replicates are retained, including the partially collapsed MNIST level-10 training replicate discussed in Appendix~\ref{app:additional-metrics}.

\begin{figure}[H]
    \centering
    \includegraphics[width=\textwidth]{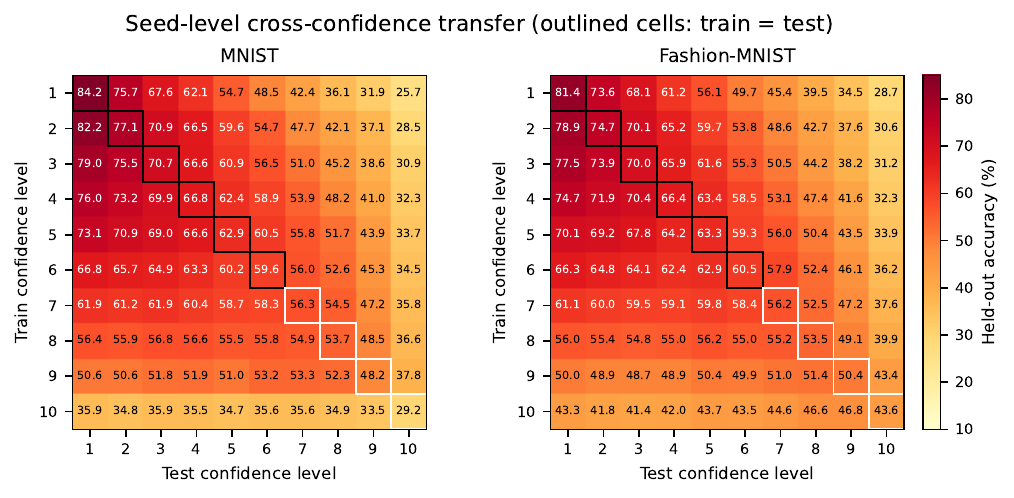}
    \caption{Seed-level cross-confidence transfer on MNIST and Fashion-MNIST. Rows indicate the confidence level used to train the probe and columns the held-out confidence level used for evaluation; level 1 is highest confidence. Cells report mean held-out accuracy over three paired replicates, and outlined cells are the same-level train--test diagonal.}
    \label{fig:cross-level-confidence-transfer}
\end{figure}

\paragraph{Sensitivity to an alternative logit-based confidence score.}
For CIFAR-10, we repeat the seed-level cross-confidence experiment after replacing the top-two logit margin with maximum softmax probability,
\begin{equation}
    c_{\mathrm{MSP}}(x)=\max_k \operatorname{softmax}(\ell(x))_k.
\end{equation}
Both variants start from the same $500{,}000$ cached trajectories, fixed labeling classifier and final-image labels. They use the same class balancing, ten confidence levels, ResNet-18 probe and optimization protocol, with $3{,}000$ training examples and $413$ validation and test examples per class and level. Each train-level-by-test-level cell is averaged over three paired replicates. Only the ranking score changes, and therefore so do the induced level memberships and data splits.

Figure~\ref{fig:cifar-confidence-score-sensitivity} shows that the two scores produce similar cross-confidence structure. In both matrices, accuracy generally decreases toward lower-confidence test levels, while the training level also affects transfer. Across the $100$ corresponding cells, the mean absolute difference is $0.91$ accuracy points and the maximum is $3.08$ points. Mean diagonal accuracy is $34.21\%$ with the logit margin and $33.84\%$ with maximum softmax probability. This is a limited sensitivity check between two related functions of the same classifier logits: it shows that the qualitative CIFAR-10 result is not specific to one of these two scores.

\begin{figure}[H]
    \centering
    \includegraphics[width=\textwidth]{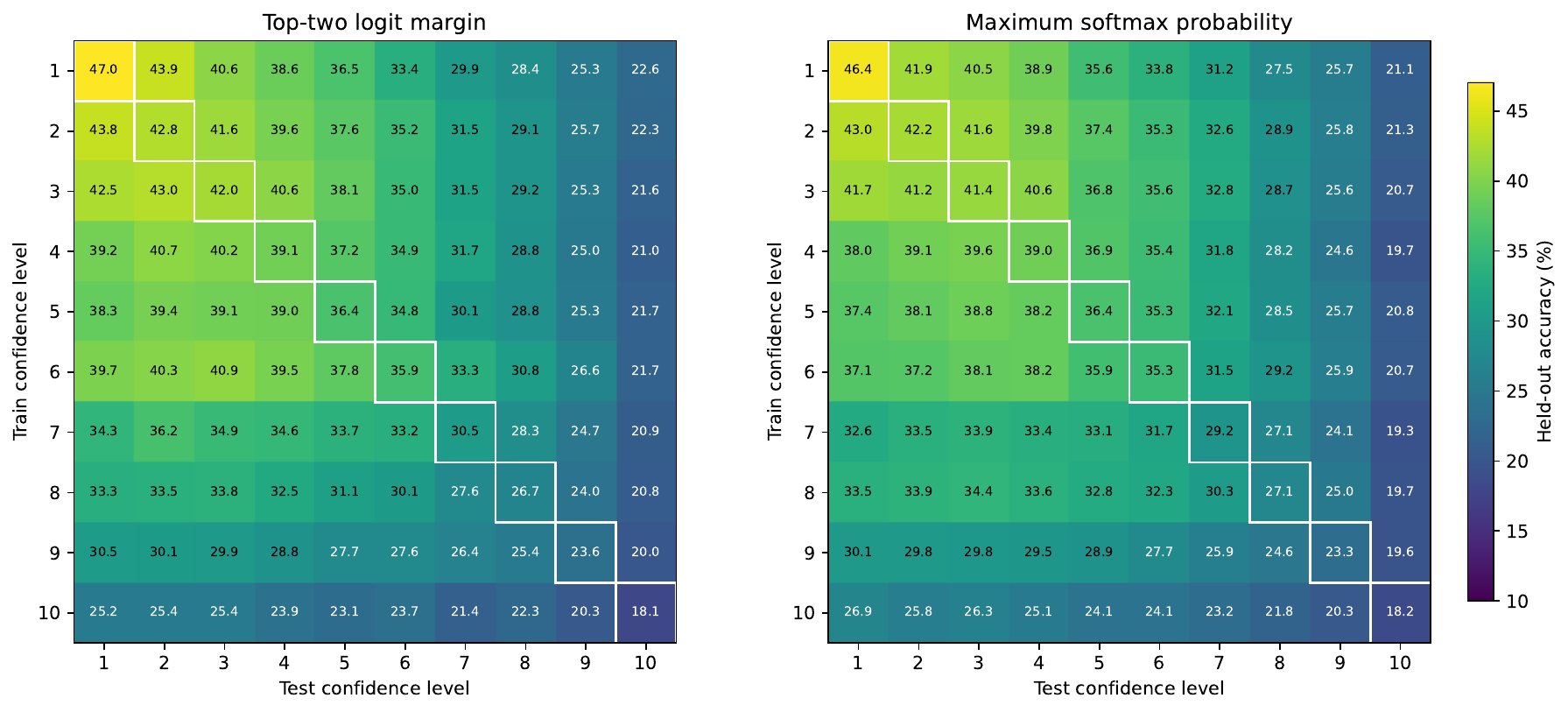}
    \caption{CIFAR-10 seed-level cross-confidence accuracy using the top-two logit margin (left) or maximum softmax probability (right) to construct the confidence levels. Rows indicate the training level and columns the held-out test level; level 1 is highest confidence. Cells report mean held-out accuracy over three paired replicates, and outlined cells are the same-level train--test diagonal. Both panels use the same color scale.}
    \label{fig:cifar-confidence-score-sensitivity}
\end{figure}

\newpage
\section{LDA--UMAP Visualization of Confidence-Stratified Seeds}
\label{app:lda-umap}

We complement the prediction results in Appendix~\ref{app:additional-confidence-results} with visualizations of MNIST seeds $X_1$ and generated images $X_0$, colored by the final-image label $Y=C(X_0)$. Within each confidence level, linear discriminant analysis (LDA) projects the flattened inputs onto nine directions that emphasize between-class variation relative to within-class variation. Uniform Manifold Approximation and Projection (UMAP) then maps these representations to two dimensions. We also apply UMAP directly to the inputs to compare the visualizations with and without the supervised projection.

\paragraph{Confidence dependence at the two endpoints.}
Figure~\ref{fig:lda-umap-endpoints} shows representative high-, medium- and low-confidence levels. The seed embeddings have distinct groups at level 1, weaker separation at level 5 and substantial mixing at level 10. Generated images show clearer separation, especially at high confidence, with some class organization remaining at the lowest level. Direct UMAP on the level-1 seeds does not reveal comparable grouping. These comparisons illustrate how the visible separation depends on both confidence and the representation supplied to the visualization.

\begin{figure}[H]
    \centering
    \includegraphics[width=0.73\textwidth]{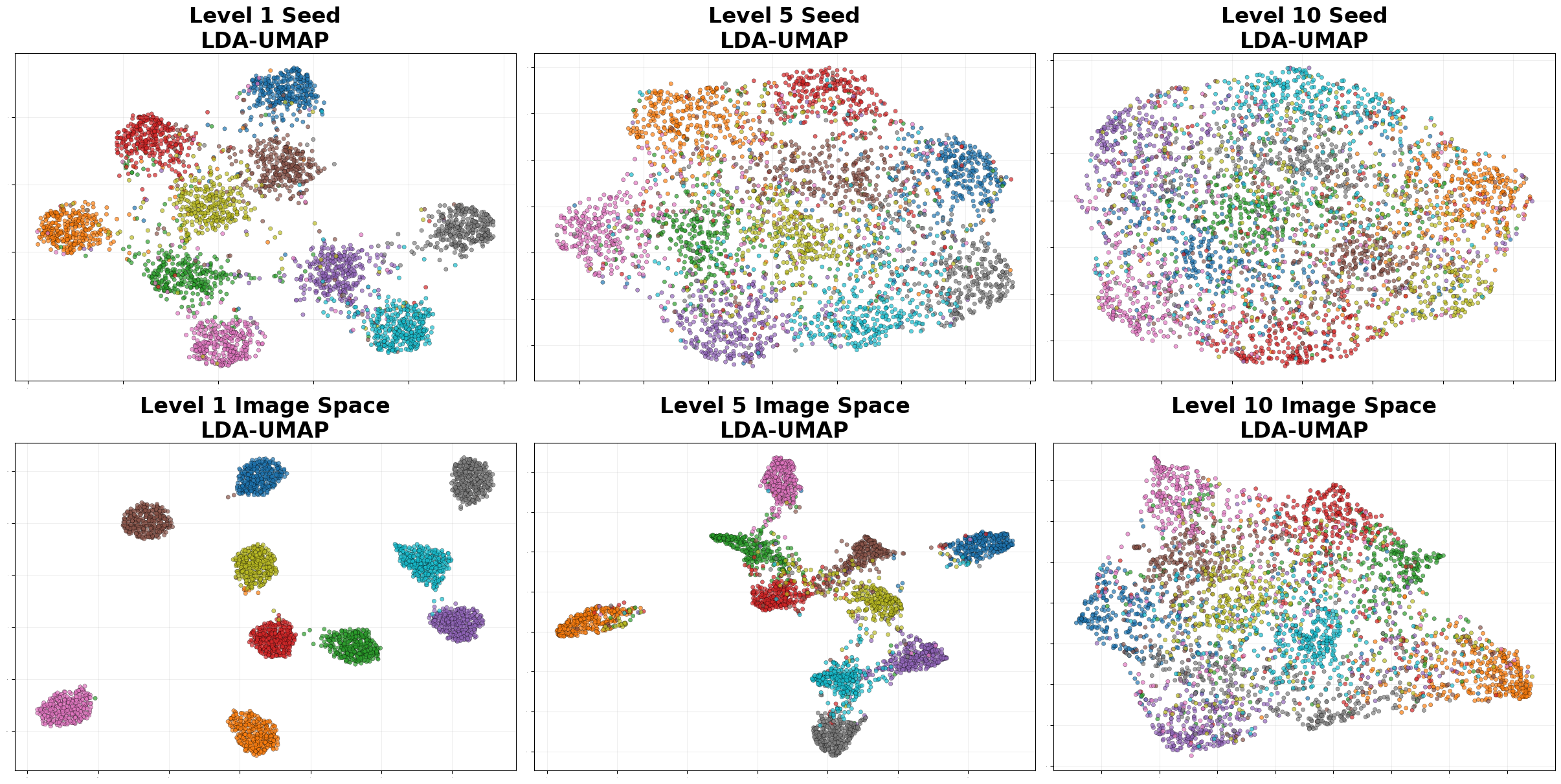}\hfill
    \includegraphics[width=0.245\textwidth]{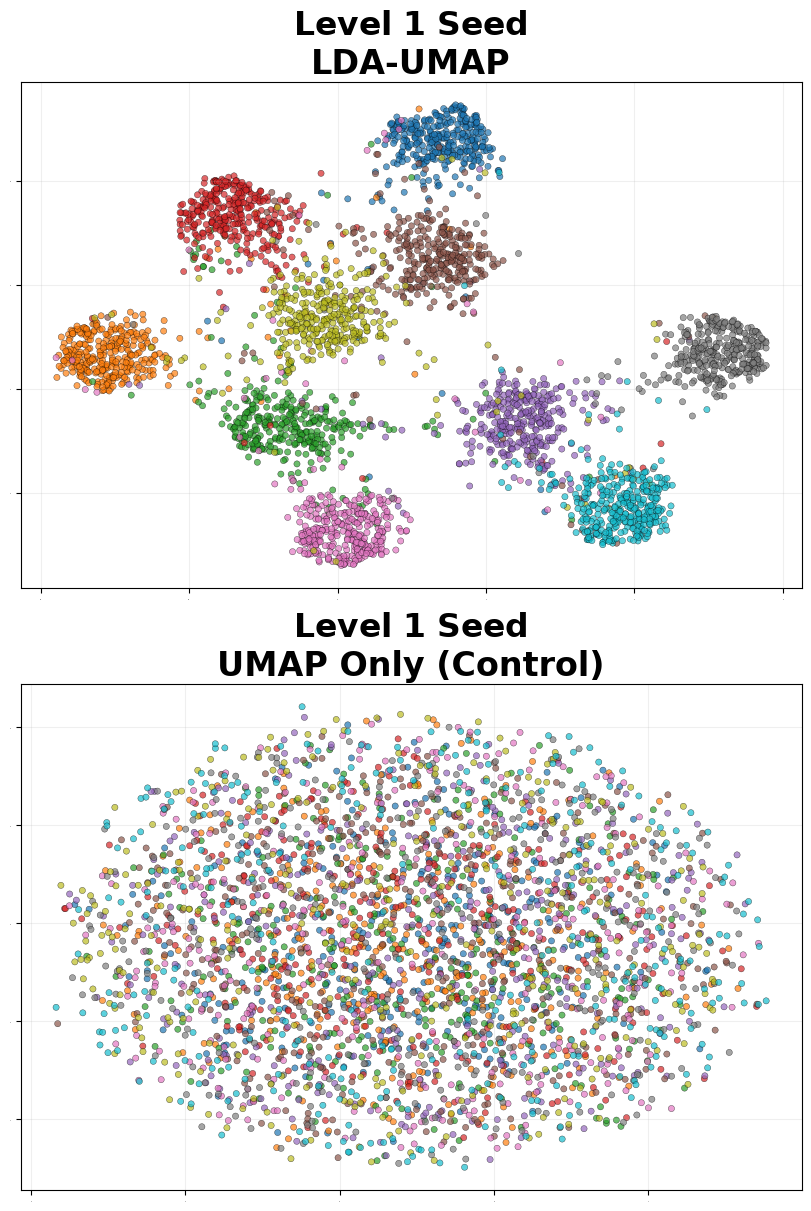}
    \caption{MNIST endpoint visualizations. Left: LDA--UMAP embeddings of seeds $X_1$ (top) and generated images $X_0$ (bottom) at confidence levels 1, 5 and 10. Right: level-1 seeds with LDA preprocessing (top) and direct UMAP (bottom). Colors denote final-image labels.}
    \label{fig:lda-umap-endpoints}
\end{figure}

\paragraph{Linear prediction and variance.}
Table~\ref{tab:lda-pca-diagnostics} reports LDA label-prediction accuracy and the variance captured by the first nine principal components of the seeds. LDA accuracy decreases across the three confidence levels while remaining above the balanced ten-class chance level of $10\%$. The PCA variance fraction stays near $2.3\%$. These quantities measure different properties: classification accuracy measures label recovery, whereas explained variance measures how much input variation a projection retains. Low PCA variance alone therefore does not imply that labels are unpredictable.

\begin{table}[H]
    \centering
    \caption{Linear diagnostics on MNIST seeds. PCA uses nine components, matching the dimension of the LDA projection.}
    \label{tab:lda-pca-diagnostics}
    \begin{tabular}{lrr}
        \hline
        Confidence level & LDA accuracy (\%) & PCA variance explained (\%) \\
        \hline
        1 (highest) & 91.91 & 2.34 \\
        5 & 77.40 & 2.31 \\
        10 (lowest) & 65.63 & 2.31 \\
        \hline
    \end{tabular}
\end{table}

\paragraph{Additional comparisons.}
Direct UMAP on the generated images also shows class grouping at high confidence, with increasing overlap at lower confidence (Figure~\ref{fig:umap-generated-images}). The image-space pattern is thus visible without LDA, although the supervised projection gives clearer separation at the lower-confidence levels shown in Figure~\ref{fig:lda-umap-endpoints}.

\begin{figure}[H]
    \centering
    \includegraphics[width=\textwidth]{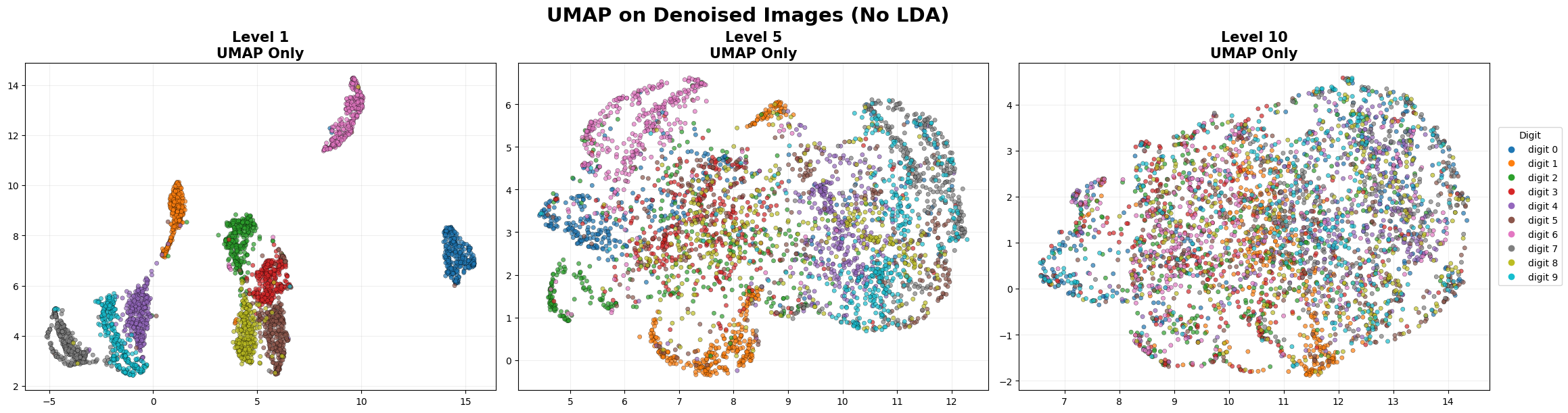}
    \caption{Direct UMAP on generated MNIST images at confidence levels 1, 5 and 10. Colors denote the same final-image labels used for LDA; labels are used only for coloring in this visualization.}
    \label{fig:umap-generated-images}
\end{figure}

Two further comparisons are shown in Figure~\ref{fig:lda-umap-comparisons}. When level-10 seeds are mapped into the embedding fitted on level 1, their density spreads across and between the level-1 groups. Under stochastic DDPM sampling, neither LDA--UMAP nor direct UMAP shows comparable class separation at the displayed levels. The latter observation is qualitatively consistent with the seed-level prediction control in Appendix~\ref{app:ddpm}.

\begin{figure}[H]
    \centering
    \begin{minipage}[t]{0.44\textwidth}
        \centering
        \textbf{(a) Confidence-level overlay}\par\smallskip
        \includegraphics[width=\linewidth]{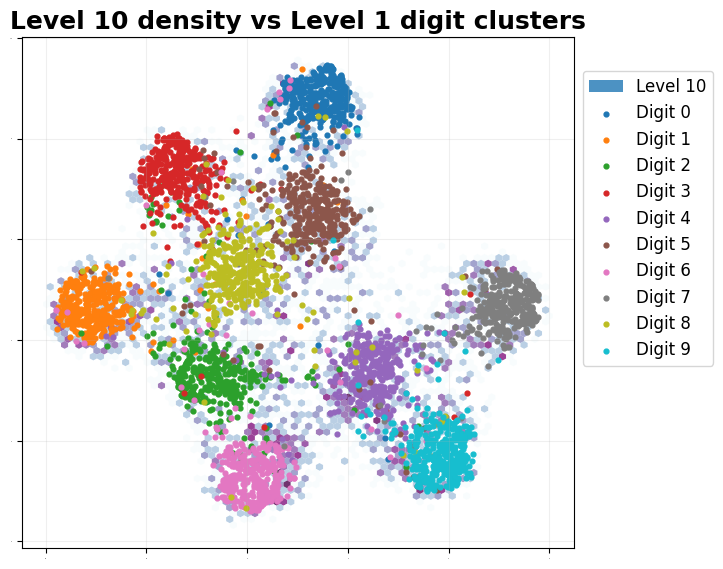}
    \end{minipage}\hfill
    \begin{minipage}[t]{0.54\textwidth}
        \centering
        \textbf{(b) Stochastic sampling}\par\smallskip
        \includegraphics[width=\linewidth]{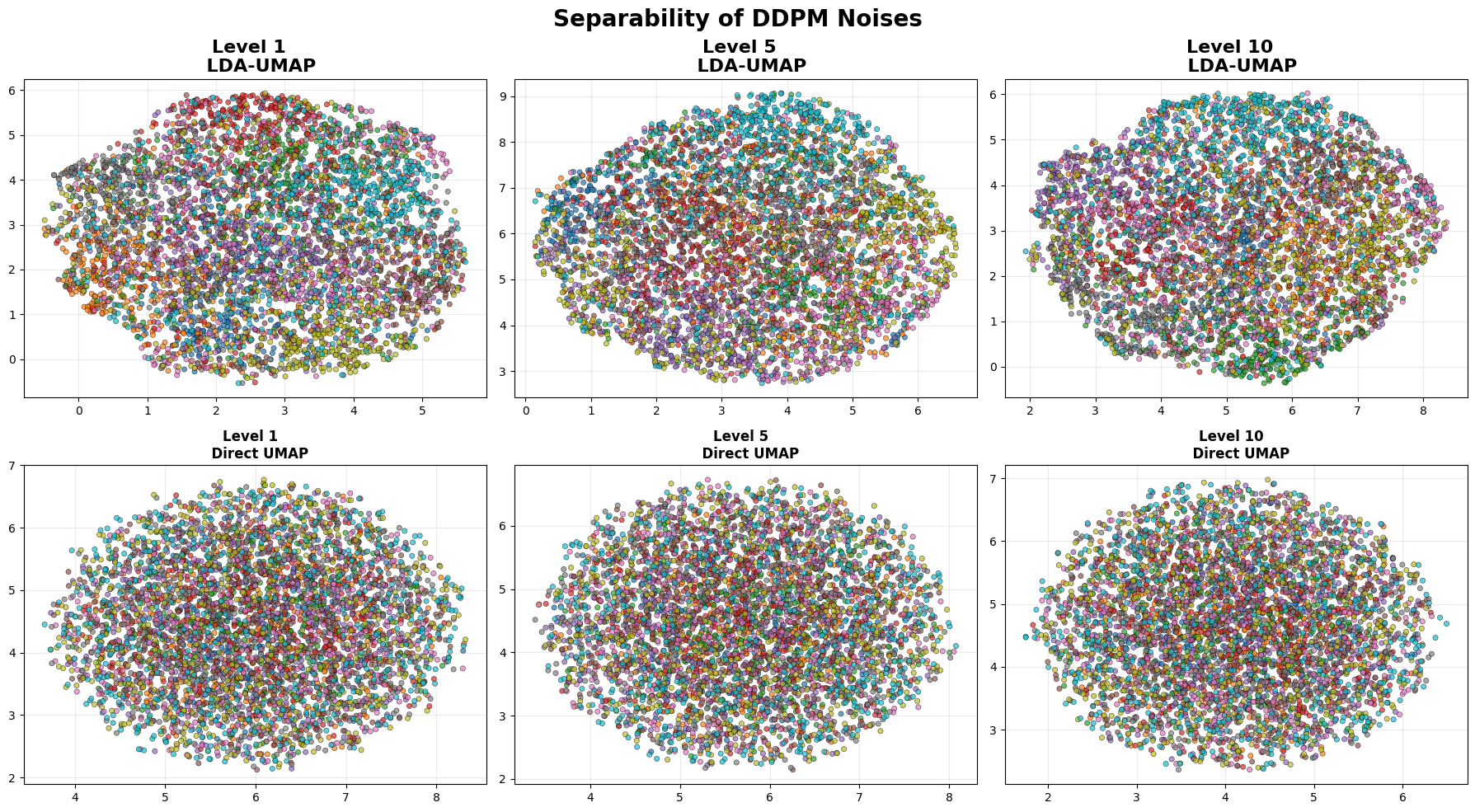}
    \end{minipage}
    \caption{Additional seed visualizations. (a) Level-10 density overlaid on the LDA--UMAP embedding fitted to level-1 DDIM seeds; colored points denote level-1 labels. (b) DDPM seeds at levels 1, 5 and 10, visualized with LDA--UMAP (top) and direct UMAP (bottom).}
    \label{fig:lda-umap-comparisons}
\end{figure}

LDA uses the final-image labels to construct the projection, so these figures are supervised diagnostics rather than evidence of unsupervised clusters. UMAP layouts also do not measure mutual information. We interpret them as qualitative complements to the held-out accessibility measurements, not as evidence that semantic information appears or disappears along deterministic trajectories.

\newpage
\section{CelebA Semantic Variable Analysis}
\label{app:celeba-analysis}

\subsection{Attribute Correlations and High-Confidence Set Overlap}
\label{app:celeba-attribute-relations}

The seven CelebA semantic properties are defined on the same 170,000 generated images and are not independent. We quantify their final-image association using the symmetric NMI between each pair of binary labeling-classifier predictions. This pairwise diagnostic is distinct from the probe NMI defined in Section~\ref{sec:probe-extracted-information}: neither attribute is treated as the prediction target, and the normalization uses the mean of the two label entropies. We separately define the high-confidence set for an attribute as all trajectories assigned to its level 1. Because levels are formed within each predicted binary value, every such set contains approximately $20\%$ of the common population. We measure pairwise set overlap by the Jaccard index; two independent sets of this size have expected Jaccard overlap $0.2^2/(2(0.2)-0.2^2)=0.111$.

Figure~\ref{fig:celeba-attribute-dependence} shows that the strongest high-confidence overlaps occur between attributes whose final labels are also strongly associated. Wearing Lipstick and Heavy Makeup have label NMI $0.716$ and level-1 Jaccard overlap $0.664$; $93.72\%$ of their binary labels agree. Wearing Lipstick and Male have label NMI $0.698$ and Jaccard overlap $0.460$, with $93.33\%$ of labels taking opposite values. Male and Heavy Makeup follow the same inverse pattern, with label NMI $0.558$ and Jaccard overlap $0.443$. A second associated group contains Smiling and High Cheekbones, whose label NMI is $0.492$ and Jaccard overlap is $0.389$. Across all 21 off-diagonal attribute pairs, pairwise label NMI and level-1 Jaccard overlap have Pearson correlation $0.969$.

The high-confidence overlap therefore largely tracks correlations among the attributes. Since all seven attributes are evaluated on the same images and several are strongly correlated, these results are related measurements rather than seven independent replications.

\begin{figure}[H]
    \centering
    \includegraphics[width=\textwidth]{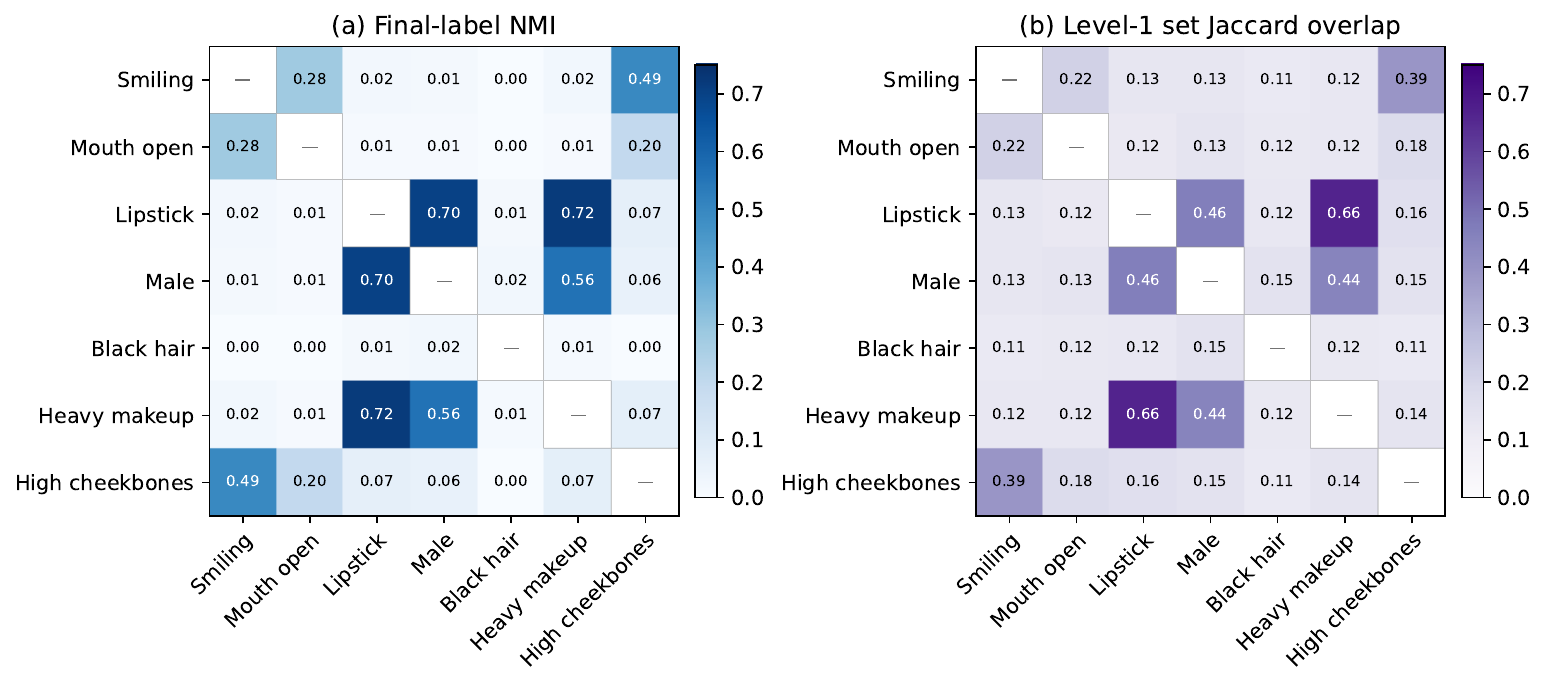}
    \caption{Dependence among the seven CelebA semantic properties on the common 170,000-image generated population. Left: symmetric NMI between final binary attribute predictions. Right: Jaccard overlap between the two level-1, class-conditionally high-confidence sets. Diagonal entries are omitted.}
    \label{fig:celeba-attribute-dependence}
\end{figure}

\subsection{Attribute Granularity via Image Degradation}
\label{sec:attribute-granularity}

Before the trajectory sweep, we used a fixed-classifier scale-retention screen to select an intentionally broad visual attribute. The screen evaluates the fixed CelebA labeling classifier against human annotations on all 19,962 official test images. Starting from its $64\times64$ input, we construct effective resolutions $32$, $16$, $8$ and $4$ using two operators. \emph{Block average} averages each non-overlapping pixel block and bilinearly resizes the result to $64\times64$. \emph{Ideal Fourier} removes frequencies at or above the target Nyquist boundary in either image direction and applies the inverse transform directly on the $64\times64$ grid. We report balanced accuracy and AUROC, together with the chance-corrected retention
\begin{equation}
    R_m(s)=\frac{m(s)-0.5}{m(64)-0.5},
\end{equation}
where $m$ is balanced accuracy or AUROC and $s$ is the effective resolution. This score is one at the original resolution and zero at chance.

Male is the strongest operator-consistent scale-retention candidate in Figure~\ref{fig:celeba-scale-retention}. Its original balanced accuracy and AUROC are $0.9789$ and $0.9978$. At effective resolution $16$, block averaging retains $0.783$ of its above-chance balanced-accuracy signal and $0.903$ of its AUROC signal; ideal Fourier filtering retains $0.850$ and $0.936$. Both operators also leave above-chance signal at resolution $8$. Black Hair retains substantial low-resolution information as well, but its two operator curves agree less closely. In contrast, the balanced-accuracy retention of Wearing Lipstick, Heavy Makeup and High Cheekbones is near zero at resolution $16$, even though some AUROC signal remains. This gap also shows that the screen measures the fixed full-resolution classifier's response to filtered inputs, not semantic scale in isolation.

The scale-retention ordering does not produce a stable ordering in the generated-trajectory results. In the level-1 trajectory curves of Figures~\ref{fig:full-celeba-accuracy-trajectories} and~\ref{fig:full-celeba-nmi-trajectories}, Black Hair first reaches $90\%$ mean accuracy at $t=0.90$, Smiling and High Cheekbones at $t=0.60$, Heavy Makeup at $t=0.55$, Wearing Lipstick and Male at $t=0.50$, and Mouth Slightly Open at $t=0.45$. At level 4, Wearing Lipstick and Male are earliest at $t=0.25$, while Black Hair never reaches the threshold. No level-5 attribute reaches $90\%$, and the final-image level-5 accuracy varies from $61.30\%$ to $86.37\%$. The ranking therefore changes with confidence level and is also affected by attribute-specific final-image probe ceilings.

This experiment supports only an operational statement: some attributes remain predictable to the fixed classifier after stronger spatial degradation than others. It does not establish a causal measure of semantic coarseness, and the trajectory data do not support a universal coarse-to-fine order of accessibility.

\begin{figure}[H]
    \centering
    \includegraphics[width=\textwidth]{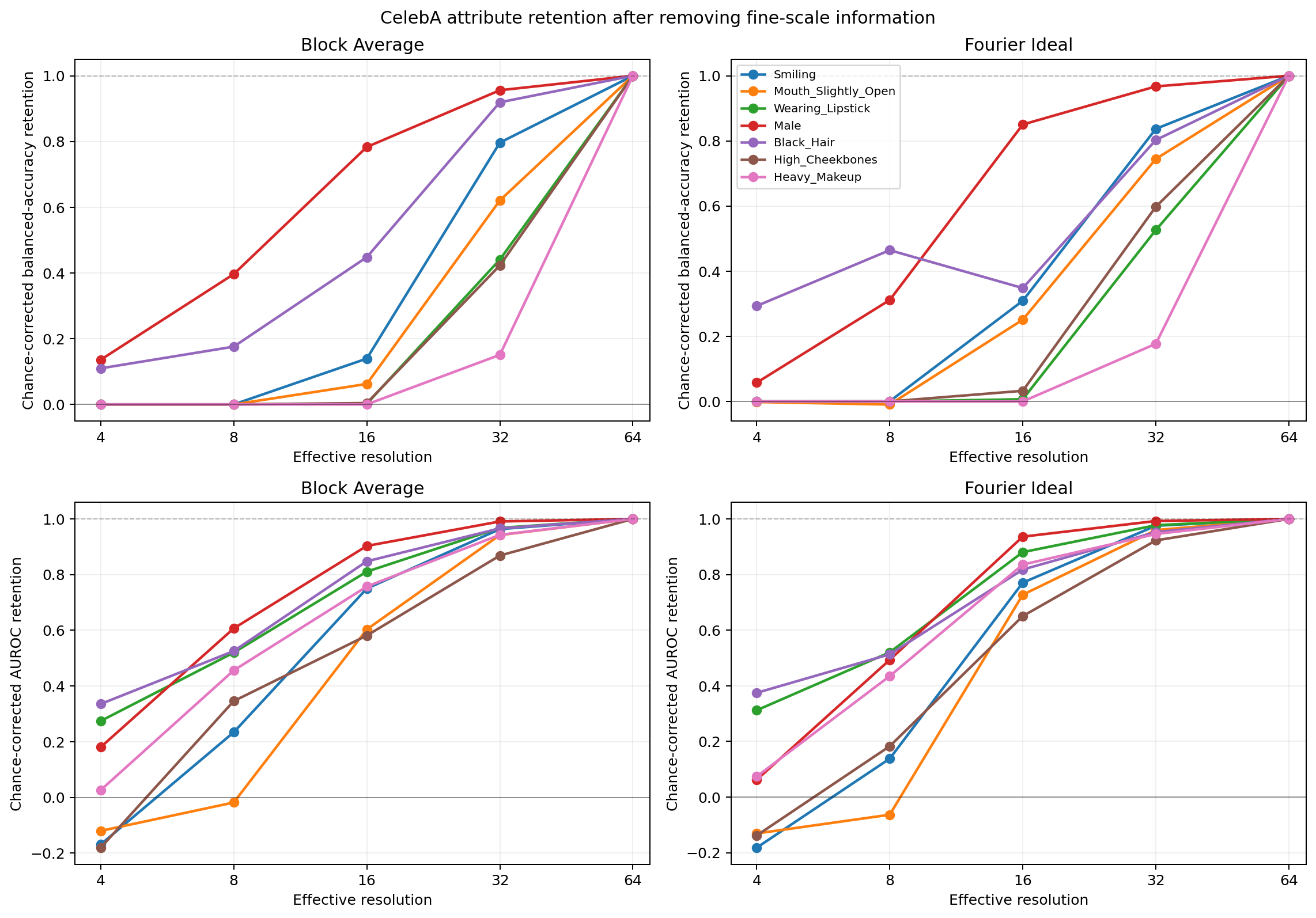}
    \caption{Chance-corrected balanced-accuracy and AUROC retention under block-average and ideal-Fourier scale restriction on the official CelebA test split. Higher retention means that more of the fixed classifier's original above-chance performance survives the removal of fine-scale image information.}
    \label{fig:celeba-scale-retention}
\end{figure}

\newpage
\section{Randomized-Target Control}
\label{app:randomized-controls}

\paragraph{Protocol.}
This control keeps the CIFAR-10 trajectory states, confidence levels, sample memberships and train/validation/test splits fixed, but independently permutes the target vector within each confidence level, replicate and split. Train, validation and test targets are permuted separately, and each permutation preserves the exact balanced class counts. For a fixed level, replicate and split, the same sample-level permutation is reused at all evaluated trajectory states. The original classifier-assigned labels therefore continue to define the frozen confidence levels and splits, but all probe training, checkpoint selection and evaluation use the randomized targets. We use the same ResNet-18 probe and optimization protocol as in the main experiment, with $30{,}000$ training examples and $4{,}130$ validation and test trajectories per condition. The grid contains ten confidence levels, five trajectory states at $t\in\{1,0.75,0.5,0.25,0\}$, and three paired replicates, for $150$ completed cells.

\paragraph{Held-out results.}
Across all 150 evaluations, mean held-out accuracy is $10.03\%$ against a balanced $10\%$ chance level, and mean NMI is $0.00409$. After averaging the three replicates within each level and trajectory state, accuracy ranges from $9.60\%$ to $10.45\%$ and NMI from $0.00286$ to $0.00514$. The expected finite-sample NMI floor from Eq.~\eqref{eq:nmi-floor} is $0.00426$ for $K=10$ and $n=4{,}130$, so the randomized-target values sit at the independence baseline throughout the grid (Figure~\ref{fig:randomized-target-heldout}).

\begin{figure}[H]
    \centering
    \includegraphics[width=\textwidth]{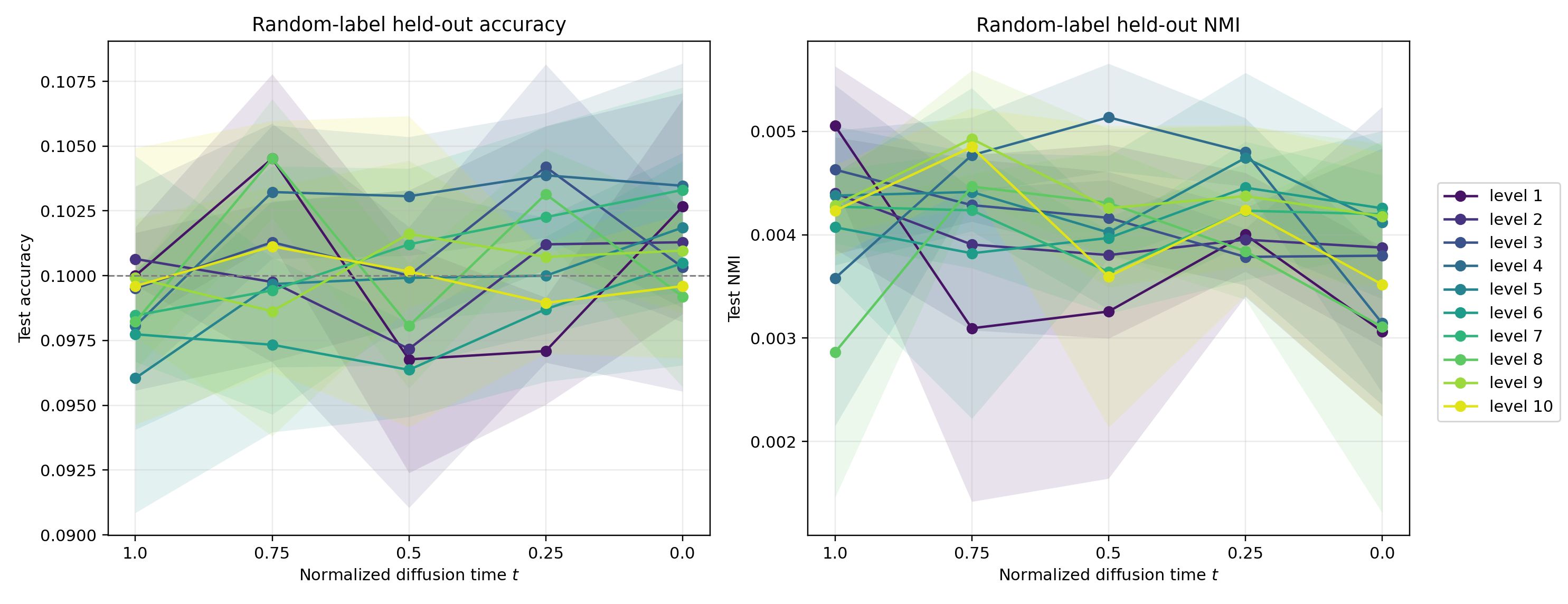}
    \caption{Held-out CIFAR-10 accuracy and NMI under randomized targets. Curves and shaded regions show the mean and population standard deviation over three paired replicates for each confidence level; each evaluation uses $n=4{,}130$ balanced trajectories. The dashed line in the accuracy panel marks the $10\%$ chance level.}
    \label{fig:randomized-target-heldout}
\end{figure}

The probes nevertheless have sufficient capacity to fit arbitrary targets. Across the 150 cells, mean training accuracy is $76.04\%$, the median is $95.14\%$, and the range is $10.45\%$--$99.59\%$; 85 cells reach at least $90\%$ training accuracy. Figure~\ref{fig:randomized-target-training} shows that high training accuracy occurs across confidence levels and trajectory states, but does not yield above-chance held-out performance.

\begin{figure}[H]
    \centering
    \includegraphics[width=0.72\textwidth]{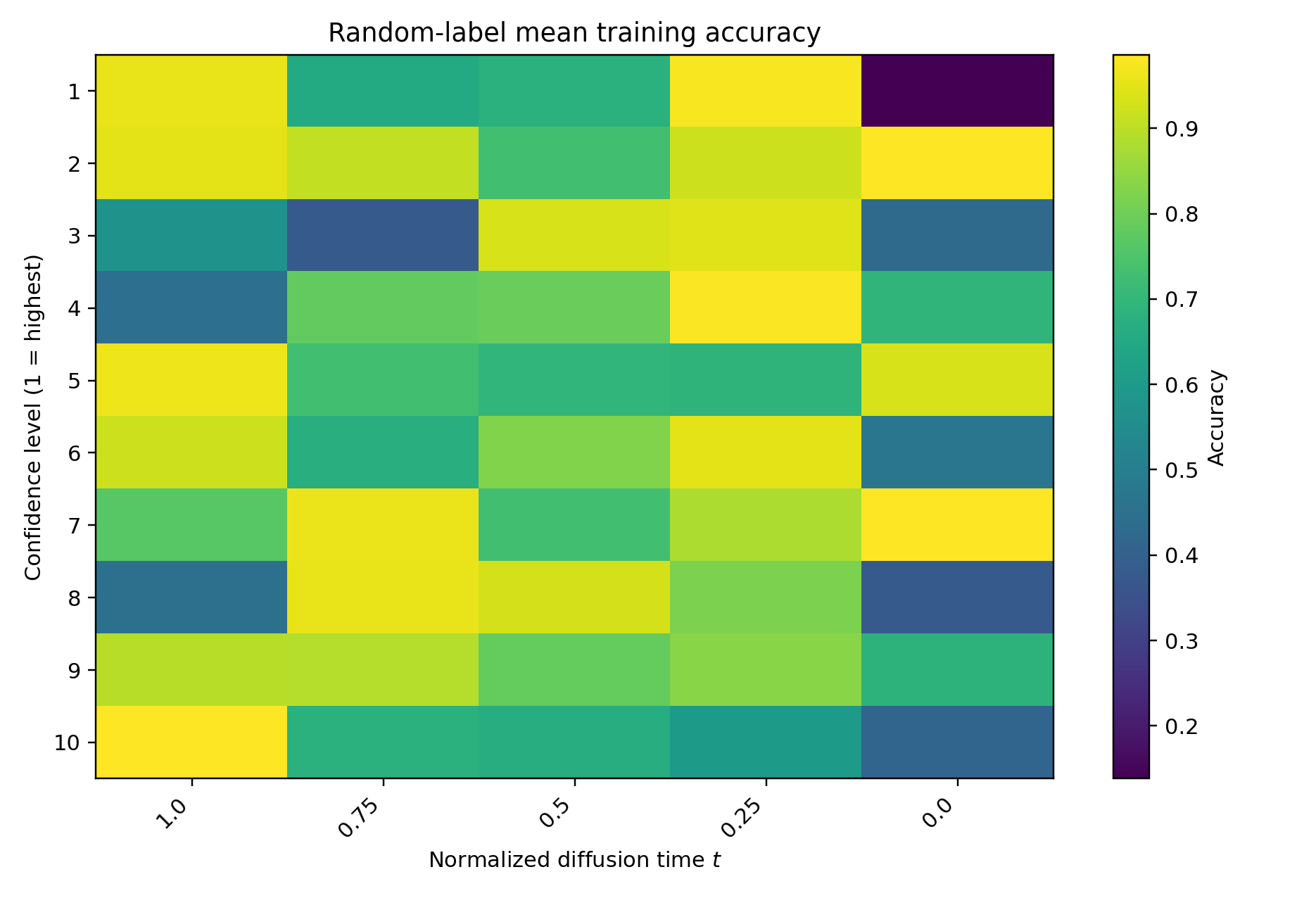}
    \caption{Mean randomized-target training accuracy over three paired replicates. The variation reflects differences in how fully the probes fit the arbitrary training assignments before validation-based early stopping.}
    \label{fig:randomized-target-training}
\end{figure}

\paragraph{Selectivity relative to the original targets.}
For each matched confidence level, trajectory state and replicate, we define selectivity as the original-target held-out metric minus the randomized-target metric. Table~\ref{tab:randomized-target-summary} averages over the 150 matched evaluations. Original-target accuracy exceeds randomized-target accuracy by $40.47$ percentage points on average, and original-target NMI exceeds randomized-target NMI by $0.3276$. Across the 50 level-by-state cells after averaging replicates, accuracy selectivity ranges from $8.14$ to $88.05$ percentage points and NMI selectivity from $0.0401$ to $0.9546$; every cell is positive (Figure~\ref{fig:randomized-target-selectivity}).

\begin{table}[H]
    \centering
    \caption{CIFAR-10 original-target and randomized-target performance over the same 150 level-by-state-by-replicate evaluations. Selectivity is computed within each matched evaluation before averaging.}
    \label{tab:randomized-target-summary}
    \begin{tabular}{lccc}
        \hline
        Metric & Original targets & Randomized targets & Selectivity \\
        \hline
        Accuracy & $50.50\%$ & $10.03\%$ & $40.47$ points \\
        NMI & $0.3316$ & $0.00409$ & $0.3276$ \\
        \hline
    \end{tabular}
\end{table}

\begin{figure}[H]
    \centering
    \includegraphics[width=\textwidth]{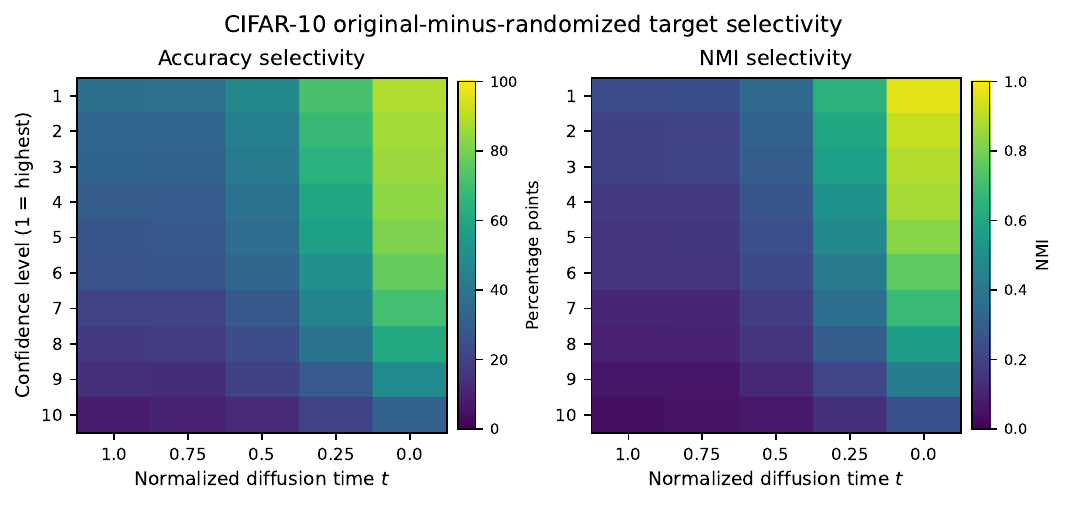}
    \caption{Per-condition selectivity under the randomized-target control. Each cell is the mean original-minus-randomized held-out metric over three paired replicates, matched by confidence level and trajectory state.}
    \label{fig:randomized-target-selectivity}
\end{figure}

This control shows that the held-out accessibility pattern is not reproduced when the same cached trajectory states are paired with arbitrary targets, even when the probe fits those targets on the training set.

\newpage
\section{Stochastic-Sampling (DDPM) Control}
\label{app:ddpm}
The deterministic-map interpretation in Section~\ref{sec:samplers} requires each seed to determine a single final image. Under stochastic sampling, the final image also depends on fresh noise injected during the reverse process, which a probe observing only $X_1$ does not receive. We therefore repeat the MNIST seed-level experiment with stochastic DDPM sampling to test whether the confidence structure observed under DDIM persists when this deterministic association is removed.

\paragraph{Matched sampling and probe setup.}
We reuse the MNIST initial seeds, frozen diffusion network, preprocessing and timestep grid from Appendix~\ref{app:experimental-details}. DDPM replaces the deterministic updates with interval-posterior mean and variance updates and independent Gaussian innovations, corresponding to $\eta=1$ rather than $\eta=0$~\citep{song2021denoising}. The endpoint convention in Table~\ref{tab:timesteps} also gives the final DDPM transition nonzero variance. We generate one stochastic realization per initial seed.

Targets and confidences are recomputed from the DDPM final images using the fixed MNIST labeling classifier (Appendix~\ref{app:architectures}). Applying the class-balancing and confidence-stratification procedure of Appendix~\ref{app:experimental-details} retains $451{,}870$ trajectories. As in DDIM, we reserve $10\%$ per class and level for validation and $10\%$ for testing, giving $4{,}520$ examples in each held-out split per DDPM level. The two samplers share initial seeds but have different level and split memberships because their final labels and confidences differ.

Probe architecture, inputs, training and checkpoint selection follow the MNIST protocol in Appendix~\ref{app:architectures}, with the training budget in Table~\ref{tab:population-splits}. Three paired replicates per training level give $30$ probes, each evaluated on all ten fixed test levels as in Appendix~\ref{app:additional-confidence-results}.

\paragraph{Results and scope.}
Figure~\ref{fig:mnist-ddpm-control} contrasts the two cross-level matrices. Under DDPM, mean NMI over the three replicates ranges from $0.000044$ to $0.000866$ across all $100$ train/test-level cells, and mean accuracy from $9.78\%$ to $10.18\%$, close to the balanced chance level of $10\%$. There is no confidence gradient comparable to DDIM: on the DDIM diagonal, NMI decreases from $0.678$ at level $1$ to $0.125$ at level $10$, while accuracy decreases from $84.19\%$ to $29.21\%$. The DDPM matrices are therefore approximately flat across confidence levels under this probe setup, with accuracy near chance and NMI near zero. Some fitted probes make constant predictions, for which empirical NMI is exactly zero even at finite sample size; the full-support approximation in Eq.~\eqref{eq:nmi-floor} does not apply to those predictions.

\begin{figure}[H]
    \centering
    \textbf{(a) Deterministic DDIM, $100$ steps}\par\smallskip
    \includegraphics[width=\textwidth]{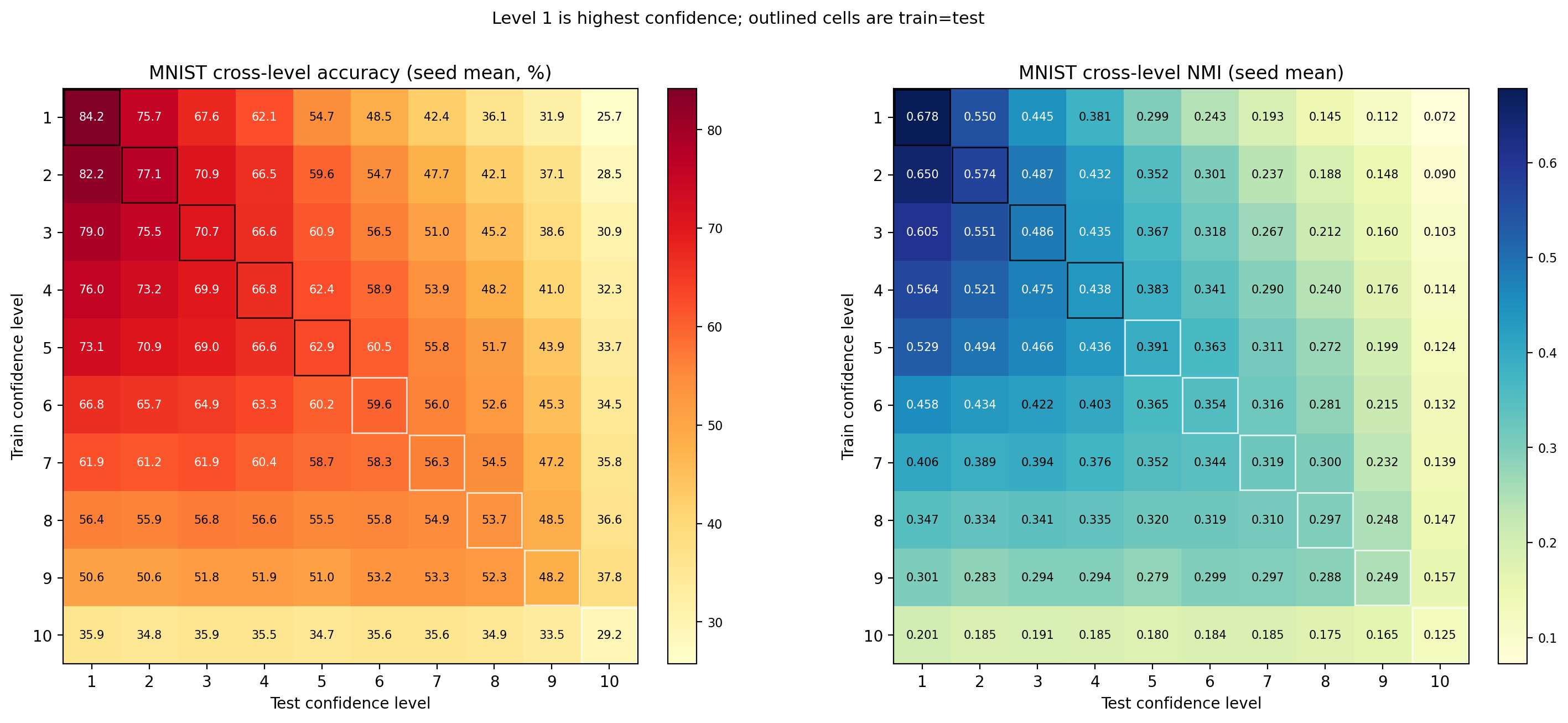}
    \par\medskip
    \textbf{(b) Stochastic DDPM, $100$ steps}\par\smallskip
    \includegraphics[width=\textwidth]{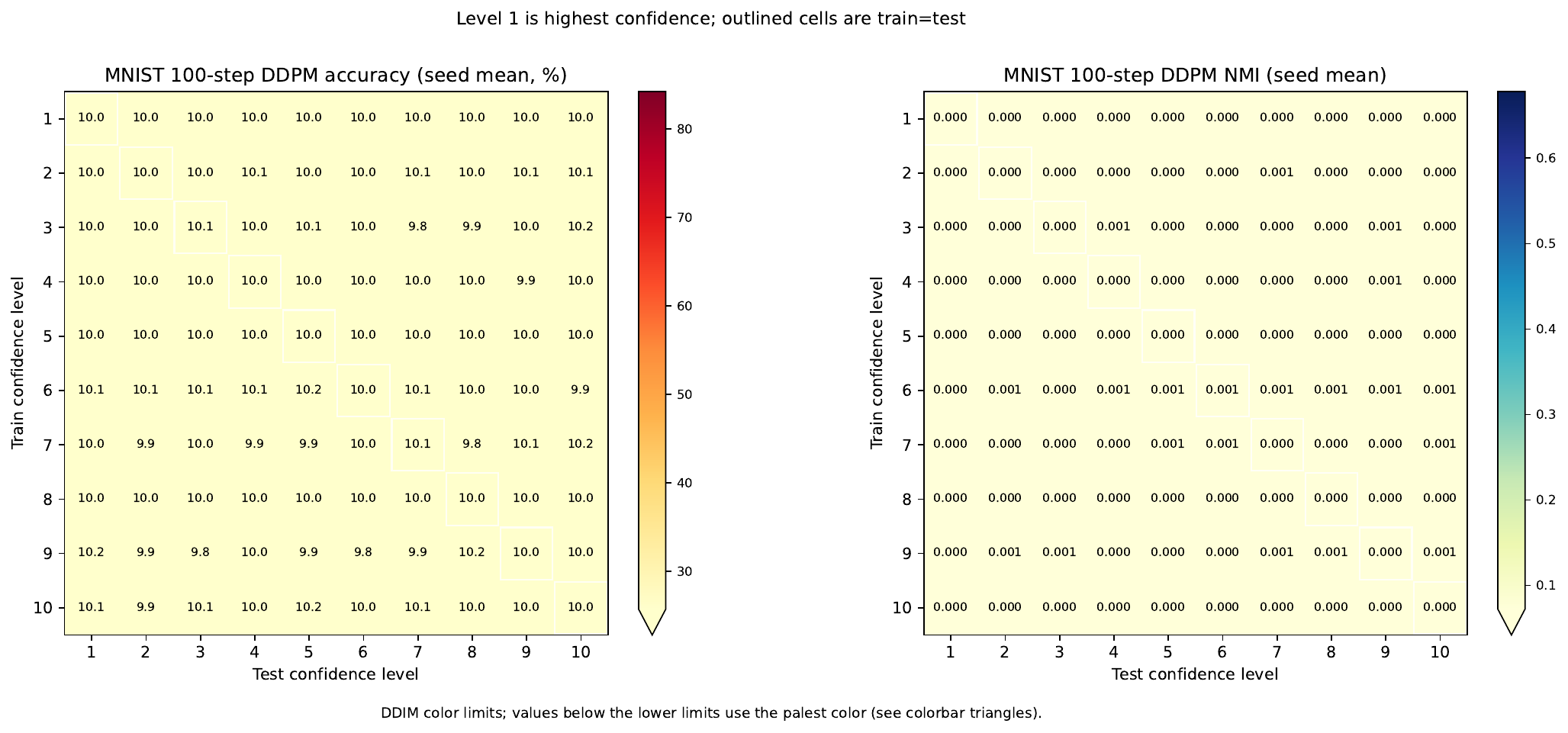}
    \caption{MNIST seed-level ($X_1$) accuracy (left) and NMI (right) under DDIM and DDPM. Rows index the training confidence level and columns the test level; level $1$ is highest confidence, and outlined cells use the same train and test level. Each entry is the mean of three probe replicates (``seed mean'' in the plot titles), evaluated on $n=4{,}540$ held-out examples per level for DDIM and $n=4{,}520$ for DDPM. The DDPM panels use exactly the DDIM color limits: $25.72$--$84.19\%$ for accuracy and $0.0722$--$0.6779$ for NMI, rounded here. All DDPM values lie below the respective lower limits and receive the palest color, as indicated by the colorbar extensions; numerical annotations are unchanged. Confidence strata and test memberships are constructed separately for each sampler.}
    \label{fig:mnist-ddpm-control}
\end{figure}

\end{document}